\documentclass[letterpaper]{article}
\usepackage{uai2020}
\usepackage[margin=1in]{geometry}

\usepackage{times}
\usepackage{hyperref}       
\usepackage{url}            
\usepackage{booktabs}       
\usepackage{amsfonts}       
\usepackage{nicefrac}       
\usepackage{color}
\usepackage{natbib}
\usepackage{microtype}      
\usepackage{amsmath, amsthm, amssymb, amsfonts, mathtools, graphicx, enumerate}
\usepackage{algorithm}
\usepackage{pifont}
\usepackage[noend]{algorithmic}
\DeclareMathOperator*{\argmax}{arg\,max}
\DeclareMathOperator*{\argmin}{arg\,min}
\newcommand{\R}{\mathbb{R}}
\newcommand{\dotp}[2]{\langle {#1}, {#2} \rangle}
\newcommand{\C}[1]{\mathcal{#1}}
\newcommand{\B}[1]{\mathbb{#1}}
\newcommand{\E}{\mathbb{E}}
\newcommand{\Scal}{\mathcal{S}}
\newcommand{\Acal}{\mathcal{A}}

\newtheorem{theorem}{Theorem}[section]
\newtheorem{corollary}[theorem]{Corollary}

\newtheorem{assumption}[theorem]{Assumption}

\newtheorem{example}[theorem]{Example}
\newtheorem{lemma}[theorem]{Lemma}

\newcount\Comments  
\definecolor{darkred}{rgb}{0.7,0,0}
\definecolor{teal}{rgb}{0.3,0.8,0.8}
\newcommand{\kibitz}[2]{\ifnum\Comments=1{\textcolor{#1}{\textsf{\footnotesize #2}}}\fi}

\title{No-regret Exploration in Contextual Reinforcement Learning}

\author{ {\bf Aditya Modi} \\
Computer Science and Engineering \\
University of Michigan\\
\And
{\bf Ambuj Tewari}  \\
Department of Statistics         \\
University of Michigan
}

\begin{document}

\maketitle

\begin{abstract}
We consider the recently proposed reinforcement learning (RL) framework of Contextual Markov Decision Processes (CMDP), where the agent interacts with a (potentially adversarial) sequence of episodic tabular MDPs. In addition, a context vector determining the MDP parameters is available to the agent at the start of each episode, thereby allowing it to learn a context-dependent near-optimal policy. In this paper, we propose a no-regret online RL algorithm in the setting where the MDP parameters are obtained from the context using generalized linear mappings (GLMs). We propose and analyze optimistic and randomized exploration methods which make (time and space) efficient online updates. The GLM based model subsumes previous work in this area and also improves previous known bounds in the special case where the contextual mapping is linear. In addition, we demonstrate a generic template to derive confidence sets using an online learning oracle and give a lower bound for the setting.
\end{abstract}

\section{INTRODUCTION}
Recent advances in reinforcement learning (RL) methods have led to increased focus on finding practical RL applications. RL algorithms provide a set of tools for tackling sequential decision making problems with potential applications ranging from web advertising and portfolio optimization, to healthcare applications like adaptive drug treatment. However, despite the empirical success of RL in simulated domains such as boardgames and video games, it has seen limited use in real world applications because of the inherent trial-and-error nature of the paradigm. In addition to these concerns, for the applications listed above, we have to essentially design adaptive methods for a {\em population} of users instead of a single system. For instance, optimizing adaptive drug treatment plans for an influx of patients has two key requirements: (1) ensure quickly learning good policies for each user and (2) share the observed outcome data efficiently across patients. Intuitively, we expect that frequently seen patient types (with some notion of similarity) can be adequately dealt with by using adaptive learning methods whereas difficult and rare cases could be carefully referred to experts to safely generate more data.

An efficient and plausible way to incorporate this heterogeneity is to include any distinguishing exogenous factors in form of a contextual information vector in the learning process. This information can include demographic, genomic features or individual measurements taken from lab tests. We model this setting using the framework of Contextual Markov Decision Processes (CMDPs) \citep{modi2018markov} where the learner has access to some {\em contextual features} at the start of every patient interaction. Similar settings have been studied with slight variations by \cite{abbasi2014online, hallak2015contextual} and \cite{dann2019policy}. While the framework proposed in these works is innovative, there are a number of deficiencies in the available set of results. First, theoretical guarantees (PAC-style mistake bounds or regret bounds) sometimes hold only under a linearity assumption on the mapping between contexts and MDPs. This assumption is quite restrictive as it enforces additional constraints on the context features which are harder to satisfy in practice. Second, if non-linear mappings are introduced \citep{abbasi2014online}, the next state distributions are left un-normalized and therefore do not correctly model the context dependence of MDP dynamics.

We address these deficiencies by considering generalized linear models (GLMs) for mapping context features to MDP parameters (succinctly referred to as GLM-CMDP). We build upon the existing work on generalized linear bandits \citep{zhang2016online} and propose UCRL2 (optimistic) and RLSVI (randomized) like algorithms with regret analyses. Overall, our contributions are as follows:
\begin{itemize}
    \item We provide optimistic and randomized regret minimizing algorithms for GLM-CMDPs. Our model subsumes/corrects previous CMDP frameworks and our analysis improves on the existing regret bounds by a factor of $\C{O}(\sqrt{S})$ in the linear case.
    \item The proposed algorithms use {\em efficient online updates}, both in terms of memory and time complexity, improving over typical OFU approaches whose running time scales linearly with number of rounds.
    \item We prove a regret lower bound for GLM-CMDP when a logistic or quadratic link function is used.
    \item We provide a generic way to convert any online no-regret algorithm for estimating GLM parameters to confidence sets. This allows an improvement in the regret incurred by our methods when the GLM parameters have additional structure (e.g., sparsity). 
\end{itemize}

\section{SETTING AND NOTATION}
\label{sec:notation}
We consider episodic Markov decision processes, denoted by tuple $(\Scal, \Acal, P, R, H)$ where $\Scal$ and $\Acal$ are finite state and action spaces, $P(\cdot|s,a)$ the transition distribution, $R(s,a)$ the reward function with mean $r(s,a)$ and $H$ is the horizon. Without loss of generality, we will consider a fixed start state for each episode. In the contextual MDP setting \citep{hallak2015contextual,modi2018markov}, the agent interacts with a sequence of MDPs $M_k$ (indexed by $k$) whose dynamics and reward functions (denoted by $P_k$ and $R_k$) are determined by an observed context vector $x_k \in \C{X}$. For notation, we use $(s_{k,h},a_{k,h}, r_{k,h},s_{k,h+1})$ to denote the transition at step $h$ in episode $k$. We denote the size of MDP parameters by the usual notation: $|\Scal| = S$ and $|\Acal| = A$.

The value of a policy in an episode $k$ is defined as the expected total return for $H$ steps in MDP $M_k$:
\begin{equation*}
v^\pi_k = \B{E}_{M_k, \pi}\Big[ \sum_{h=1}^H r_{k,h}\Big]    
\end{equation*}
The optimal policy for episode $k$ is denoted by $\pi^*_k \coloneqq \argmax_{\pi} v^\pi_k$ and its value as $v^*_k$. The agent's goal in the CMDP setting is to learn a context dependent policy $\pi: \C{X} \times \C{S} \rightarrow \Acal$ such that cumulative expected return over $K$ episodes is maximized. We quantify the agent's performance by the total regret incurred over a (potentially adversarial) sequence of $K$ contexts:
\begin{align}
    R(K) \coloneqq \sum_{k=1}^K v^*_k - v^{\pi_k}_k
\end{align}
Note that the regret here is defined with respect to the sequence of context dependent optimal policies.

\paragraph{Additional notation.} For two matrices $X$ and $Y$, the inner product is defined as $\langle X,Y \rangle \coloneqq \text{Tr}(X^\top Y)$. For a vector $x \in \R^d$ and a matrix $A \in \R^{d \times d}$, we define $\|x\|_A^2 \coloneqq x^\top A x$. For matrices $W \in \R^{m \times n}$ and $X \in \R^{n \times n}$, we define $\|W\|^2_X \coloneqq \sum_{i=1}^m \|W^{(i)}\|^2_X $ where $W^{(i)}$ is the $i^{\text{th}}$ row of the matrix. Further, we reserve the notation $\|W\|_F$ to denote the Frobenius norm of a matrix $W$. 
Any norm which appears without a subscript will denote the $\ell_2$ norm for a vector and the Frobenius norm for a matrix.

\subsection{GENERALIZED LINEAR MODEL FOR CMDPs}
Using a linear mapping of the predictors is a simple and ubiquitous approach for modeling contextual/dynamical dependence in sequential decision making problems. Linear models are also well known for being interpretable and explainable, properties which are very valuable in our motivating settings. Similarly, we also utilize this structural simplicity of linearity and model the categorical output space ($p(\cdot|s,a)$) in a contextual MDP using generalized linear mappings. Specifically, for each pair $s,a \in \Scal \times \Acal$, there exists a weight matrix $W_{sa} \in \C{W} \subseteq \R^{S \times d}$ where $\C{W}$ is a convex set. For any context $x_k \in \R^d$, the next state distribution for the pair is specified by a GLM:
\begin{equation}
P_k(\cdot|s,a) = \nabla \Phi (W_{sa}x_k) 
\label{eq:glm_map}
\end{equation}
where $\Phi(\cdot) : \R^S \rightarrow \R$ is the link function of the GLM\footnote{We abuse the term GLM here as we don't necessarily consider a complementary exponential family model in eq.~\eqref{eq:glm_map}}. We will assume that this link function is convex which is always the case for a canonical exponential family \citep{lauritzen1996graphical}. For rewards, we assume that each mean reward is given by a linear function\footnote{ Similar results can be derived for GLM reward functions.} of the context: $r_k(s,a) \coloneqq \theta^{\top}_{sa}x_k$ where $\theta \in \Theta \subseteq \R^d$. In addition, we will make the following assumptions about the link function.

\begin{assumption}
The function $\Phi(\cdot)$ is $\alpha$-strongly convex and $\beta$-strongly smooth, that is:
\begin{align}
    \Phi(v) \geq {} & \Phi(u) + \langle \nabla \Phi(u), v-u \rangle + \tfrac{\alpha}{2} \|u-v\|^2_2\\
    \Phi(v) \leq {} & \Phi(u) + \langle \nabla \Phi(u), v-u \rangle + \tfrac{\beta}{2} \|u-v\|^2_2
\end{align}
\end{assumption}

We will see that this assumption is critical for constructing the confidence sets used in our algorithm. We make another assumption about the size of the weight matrices $W^*_{sa}$ and contexts $x_k$:
\begin{assumption}
\label{as:compact}
For all episodes $k$, we have $\|x_k\|_2 \leq R$ and for all state-action pairs $(s,a)$, $\|W^{(i)}_{sa}\|_2 \leq B_p$ and $\|\theta_{sa}\|_2 \leq B_r$. So, we have $\|Wx_k\|_\infty \leq B_pR$ for all $W \in \C{W}$.
\end{assumption}

The following two contextual MDP models are special cases of our setting:
\begin{example}[Multinomial logit model, \cite{agarwal2013selective}]
\label{ex:logit}
Each next state is sampled from a categorical distribution with probabilities\footnote{Without loss of generality, we can set the last row $W^{(S)}_{sa}$ of the weight matrix to be 0 to avoid an overparameterized system.}:
\[
P_x(s_i|s,a) = \frac{\exp (W^{(i)}_{sa}x)}{\sum_{j=1}^S \exp (W^{(j)}_{sa}x)}
\]
The link function for this case can be given as $\Phi(y) = \log (\sum_{i=1}^S \exp (y_i))$ which can be shown to be strongly convex with $\alpha = \tfrac{1}{\exp{(BR)}S^2}$ and smooth with $\beta = 1$.
\end{example}

\begin{example}[Linear combination of MDPs, \cite{modi2018markov}]
\label{ex:lincomb}
Each MDP is obtained by a linear combination of $d$ base MDPs $\{(\Scal,\Acal,P^i,R^i,H)\}_{i=1}^d$. Here, $x_k \in \Delta_{d-1}$\footnote{ $\Delta_{d-1}$ denotes the simplex $\{x \in \R^d: \|x\|_1 =1,\, x \ge 0\}$.}, and $P_k(\cdot|s,a) \coloneqq \sum_{i=1}^d x_{ki} P^i(\cdot|s,a)$. The link function for this can be shown to be:
\[
\Phi(y) = \tfrac{1}{2}\|y\|^2_2
\]
which is strongly convex and smooth with parameters $\alpha = \beta = 1$. Moreover, $W_{sa}$ here is the $S \times d$ matrix containing each next state distribution in a column. We have, $B_p \leq \sqrt{d}$, $\|W_{sa}\|_F \leq \sqrt{d}$ and $\|W_{sa}x_k\|_2 \leq 1$.
\end{example}

\section{ONLINE ESTIMATES AND CONFIDENCE SET CONSTRUCTION}
\label{sec:ONS}
In order to obtain a no-regret algorithm for our setting, we will follow the popular \emph{optimism in the face of uncertainty} (OFU) approach which relies on the construction of confidence sets for MDP parameters at the beginning of each episode. We focus on deriving these confidence sets for the next state distributions for all state action pairs. We assume that the link function $\Phi$ and values $\alpha$, $B$ and $R$ are known a priori. The confidence sets are constructed and used in the following manner in the OFU template for MDPs: at the beginning of each episode $k=1,2,\ldots,K$:
\begin{itemize}
        \item For each $(s,a)$, compute an estimate of transition distribution $\widehat{P}_k(\cdot|s,a)$ and mean reward $\hat{r}_k(s,a)$ along with confidence sets $\C{P}$ and $\C{R}$ such that $P_k(\cdot|s,a) \in \C{P}$ and $r_k(s,a) \in \C{R}$ with high probability.
        \item Compute an optimistic policy $\pi_k$ using the confidence sets and unroll a trajectory in $M_k$ with $\pi_k$. Using observed transitions, update the estimates and confidence sets. 
\end{itemize}
Therefore, in the GLM-CMDP setup, estimating transition distributions and reward functions is the same as estimating the underlying parameters $W_{sa}$ and $\theta_{sa}$ for each pair $(s,a)$. Likewise, any confidence set $\C{W}_{sa}$ for $W_{sa}$ ($\Theta_{sa}$ for $\theta_{sa}$) can be translated into a confidence set of transition distributions.

In our final algorithm for GLM-CMDP, we will use the method from this section for estimating the next state distribution for each state-action pair. The reward parameter $\theta_{sa}$ and confidence set $\Theta_{sa}$ is estimated using the linear bandit estimator (\cite{lattimore_szepesvári_2020},  Chap. 20). Here, we solely focus on the following online estimation problem without any reference to the CMDP setup. Specifically, given a link function $\Phi$, the learner observes a sequence of contexts $x_t \in \C{X}$ ($t=1,2,\ldots$) and a sample $y_t$ drawn from the distribution $P_t \equiv \nabla \Phi(W^*x_t)$ over a finite domain of size $S$. Here, we use $W^*$ to denote the true parameter for the given GLM model. The learner's task is to compute an estimate $W_t$ for $W^*$ and a confidence set $\C{W}_t$ after any such $t$ samples. We frame this as an online optimization problem with the following loss sequence (based on the negative log-likelihood):
\begin{equation}
l_t(W; x_t, y_t) = \Phi(W x_t) - y_t^\top W x_t
\label{eq:loss_fn}
\end{equation}
where $y_t$ is the one-hot representation of the observed sample in round $t$. This loss function preserves the strong convexity of $\Phi$ with respect to $W x_t$ and is a proper loss function \citep{agarwal2013selective}:
\begin{equation}
\label{eq:callib}
\argmin_W \B{E}\big[ l_t(W; x_t, y_t) | x_t \big] = W^*    
\end{equation}

Since our aim is computational and memory efficiency, we carefully follow the Online Newton Step \citep{hazan2007logarithmic} based method proposed for $0/1$ rewards with logistic link function in \cite{zhang2016online}. While deriving the confidence set in this extension to GLMs, we use properties of categorical vectors in various places in the analysis which eventually saves a factor of $S$. The online update scheme is shown in Algorithm~\ref{alg:GLM-ONS}. Interestingly, note that for tabular MDPs, where $d=\alpha=1$ and $\Phi(y) = \tfrac{1}{2}\|y\|^2_2$, with $\eta=1$, we would recover the empirical average distribution as the online estimate.
\begin{algorithm}[ht]
    \caption{Online parameter estimation for GLMs}
    \begin{algorithmic}[1]
    \label{alg:GLM-ONS}
    \STATE {\bfseries Input: }$\Phi, \alpha, \eta$
    \STATE Set $W_1 \leftarrow \mathbf{0}$, $Z_1 \leftarrow \lambda \mathbb{I}_d$
    \FOR{$t=1,2,\ldots$}
    \STATE Observe $x_t$ and sample $y_t \sim P_t(\cdot)$
    \STATE Compute new estimate $W_{t+1}$:
    \begin{equation}
        \label{eq:ONS}
        \argmin_{W \in \C{W}} \tfrac{\|W - W_t\|^2_{Z_{t+1}}}{2} +\eta \langle \nabla l_t(W_t x_t) x_t^\top, W - W_t \rangle
    \end{equation}
    where $Z_{t+1} = Z_t + \frac{\eta \alpha}{2} x_t x_t^\top$.
    \ENDFOR
\end{algorithmic}
\end{algorithm}
Along with the estimate $W_{t+1}$, we can also construct a high probability confidence set as follows:
\begin{theorem}[Confidence set for $W^*$]
\label{thm:ONS}
In Algorithm~\ref{alg:GLM-ONS}, for all timesteps $t=1,2,\ldots$, with probability at least $1-\delta$, we have:
\begin{equation}
    \|W_{t+1} - W^*\|_{Z_{t+1}} \leq \sqrt{\gamma_{t+1}}
\end{equation}
where
\begin{align}
\label{eq:ONS_CI}
    \gamma_{t+1} = {} & \lambda B^2 + 8\eta B_pR \nonumber \\
    {} & + \>2\eta \Big[(\tfrac{4}{\alpha} + \tfrac{8}{3}B_pR) \tau_t + \tfrac{4}{\alpha} \log \tfrac{\det(Z_{t+1})}{\det(Z_1)} \Big]
\end{align}
with $\tau_t = \log (2 \lceil 2 \log St \rceil t^2/\delta)$ and $B = \max_{W \in \C{W}} \|W\|_F$.
\end{theorem}
Any upper bound for $\|W^*\|_F^2$ can be substituted for $B$ the confidence width in eq~\eqref{eq:ONS_CI}. The term $\gamma_t$ depends on the size of the true weight matrix, strong convexity parameter $\frac{1}{\alpha}$ and the log determinant of the covariance matrix. We will later show that the last term grows at a $\C{O}(d\log t)$ rate. Therefore, overall $\gamma_t$ scales as $\C{O}(S + \frac{d}{\alpha}\log^2 t)$. The complete proof can be found in Appendix~\ref{app:ONS}.

Algorithm~\ref{alg:GLM-ONS} only stores the empirical covariance matrix and solves the optimization problem \eqref{eq:ONS} for the current context. Since $\C{W}$ is convex, this is a tractable problem and can be solved via any off-the-shelf optimizer up to desired accuracy. The total computation time for each context and all $(s,a)$ pairs is $O(\text{poly}(S,A,d))$ with no dependence on $t$. Furthermore, we only store $SA$-many matrices of size $S \times d$ and covariance matrices of sizes $d \times d$. Thus, both time and memory complexity of the method are bounded by $\C{O}(\text{poly}(S,A,H,d))$ per episode.

\section{NO-REGRET ALGORITHMS FOR GLM-CMDP}
\subsection{OPTIMISTIC REINFORCEMENT LEARNING FOR GLM CMDP}
\label{sec:algorithm}
In this section, we describe the OFU based online learning algorithm which leverages the confidence sets as described in the previous section. Not surprisingly, our algorithm is similar to the algorithm of \cite{dann2019policy} and \cite{abbasi2014online} and follows the standard format for no-regret bounds in MDPs. In all discussions about CMDPs, we will again use $x_k \in \C{X}$ to denote the context for episode $k$ and use Algorithm~\ref{alg:GLM-ONS} from the previous section to estimate the corresponding MDP $M_k$. Specifically, for each state-action pair $(s,a)$, we use all observed transitions to estimate $W_{sa}$ and $\theta_{sa}$. We compute and store the quantities used in Algorithm~\ref{alg:GLM-ONS} for each $(s,a)$: we use $\widehat{W}_{k,sa}$ to denote the parameter estimate for $W_{sa}$ at the beginning of the $k^\text{th}$ episode. Similarly, we use the notation $\gamma_{k,sa}$ and $Z_{k,sa}$ for the other terms. Using the estimate $\widehat{W}_{k,sa}$ and the confidence set, we compute the confidence interval for $P_k(\cdot|s,a)$:
\begin{align}
    \xi^{(p)}_{k,sa} \coloneqq {} & \|P_k(\cdot|s,a) - \widehat{P}_k(\cdot|s,a)\|_1 \nonumber \\
    \leq {} & \beta \sqrt{S} \|W_{sa} - \widehat{W}_{k,sa}\|_{Z_{k,sa}}\|x_k\|_{Z_{k,sa}^{-1}} \nonumber \\
    \leq {} & \beta \sqrt{S} \sqrt{\gamma_{k,sa}}\|x_k\|_{Z_{k,sa}^{-1}} \label{eq:CI_p}
\end{align}
where in the definition of $\gamma_{k,sa}$ we use $\delta=\delta_p$. It is again easy to see that for tabular MDPs with $d=1$, we recover a similar confidence interval as used in \cite{jaksch2010near}. For rewards, using the results from linear contextual bandit literature (\cite{lattimore_szepesvári_2020}, Theorem~20.5), we use the following confidence interval: 
\begin{align}
\xi^{(r)}_{k,sa} \coloneqq {} & |r_k(s,a) - \hat{r}_k(s,a)| \nonumber \\ 
= {} & \underbrace{\left(\sqrt{\lambda d} + \sqrt{\tfrac{1}{4} \log \tfrac{\det Z_{k,sa}}{\delta_r^2 \det \lambda I}}\right)}_{\coloneqq \zeta_{k,sa}} \|x_k\|_{Z_{k,sa}^{-1}}
\label{eq:CI_r}
\end{align}
In \texttt{GLM-ORL}, we use these confidence intervals to compute an optimistic policy (Lines~\ref{algline:opt_start}-\ref{algline:opt_end}). The computed value function is optimistic as we add the total uncertainty as a bonus (Line~\ref{algline:opt_bonus}) during each Bellman backup. For any step $h$, we clip the optimistic estimate between $[0,H-h]$ during Bellman backups (Line~\ref{algline:clip}\footnote{We use the notation $a \wedge b$ to denote $\min(a,b)$ and $a\vee b$ for $\max(a,b)$.}). After unrolling an episode using $\pi_k$, we update the parameter estimates and confidence sets for every observed $(s,a)$ pair. 

For any sequence of $K$ contexts, we can guarantee the following regret bound:
\begin{theorem}[Regret of \texttt{GLM-ORL}]
\label{thm:regret}
For any $\delta \in (0,1)$, if Algorithm~\ref{alg:GLM-ORL} is run with the estimation method \ref{alg:GLM-ONS}, then for all $K \in \B{N}$ and with probability at least $1-\delta$, the regret $R(K)$ is:
\begin{equation*}
    \tilde{\C{O}} \Big( \Big( \frac{\sqrt{d}\max_{s,a}\|W_{sa}\|_F}{\sqrt{\alpha}} + \frac{d}{\alpha} \Big) \beta SH^2\sqrt{AK} \log \frac{KHd}{\lambda \delta}\Big)
\end{equation*}
\end{theorem}
If $\|W^{(i)}\|$ is bounded by $B_p$, we get $\|W_{sa}\|_F^2 \le SB_p^2$, whereas, for the linear case (Ex.~\ref{ex:lincomb}), $\|W_{sa}\|_F^2 \le \sqrt{d}$. Substituting the bounds on $\|W_{sa}\|^2_F$, we get:
\begin{corollary}[Multinomial logit model]
For example~\ref{ex:logit}, we have $\|W\|_F \leq B\sqrt{S}$, $\alpha = \frac{1}{\exp (BR) S^2}$ and $\beta = 1$. Therefore, the regret bound of Algorithm~\ref{alg:GLM-ORL} is $\tilde{\C{O}}(dS^3H^2\sqrt{AK})$.
\end{corollary}
\begin{corollary}[Regret bound for linear combination case]
\label{cor:lincomb}
For example~\ref{ex:lincomb}, with $\|W\|_F \leq \sqrt{d}$, the regret bound of Algorithm~\ref{alg:GLM-ORL} is $\tilde{\C{O}}(dSH^2\sqrt{AK})$.
\end{corollary}

\begin{algorithm}[ht]
    \caption{\texttt{GLM-ORL} (GLM Optimistic Reinforcement Learning)}
    \begin{algorithmic}[1]
    \label{alg:GLM-ORL}
    \STATE {\bfseries Input:}$\Scal,\Acal,H,\Phi,d,\C{W}$, $\lambda$, $\delta$
    \STATE $\delta' = \tfrac{\delta}{2SA + SH}$, $\widetilde{V}_{k,H+1}(s) = 0$ $\forall s \in \Scal, k \in \B{N}$
    \FOR{$k \leftarrow 1, 2, 3, \ldots$}
        \STATE Observe current context $x_k$
        \FOR{$s \in \Scal, a \in \Acal$}
            \STATE $\widehat{P}_k(\cdot|s,a) \leftarrow \nabla \Phi (\widehat{W}_{k,sa} x_k)$
            \STATE $\hat{r}_k(s,a) \leftarrow \dotp{\hat{\theta}_{k,sa}}{x_k} $
            \STATE Compute conf. intervals using eqns. \eqref{eq:CI_p}, \eqref{eq:CI_r}
        \ENDFOR
        \FOR{$h \leftarrow H,H-1,\cdots, 1$, and $s \in \Scal$}\label{algline:opt_start}
            \FOR{$a \in \Acal$}
                \STATE $\varphi = \|\widetilde{V}_{k,h+1}\|_\infty \xi^{(p)}_{k,sa} + \xi^{(r)}_{k,sa}$ \label{algline:opt_bonus}
                \STATE $\widetilde{Q}_{k,h} (s,a) =  \widehat{P}_{k,sa}^\top\widetilde{V}_{k,h+1} + \hat{r}_k(s,a) + \varphi$
                \STATE $\widetilde{Q}_{k,h} (s,a) = 0 \vee (\widetilde{Q}_{k,h} (s,a) \wedge V_h^{\text{max}})$ \label{algline:clip}
            \ENDFOR
            \STATE $\pi_{k,h}(s) = \argmax_a \widetilde{Q}_{k,h}(s,a)$
            \STATE $\widetilde{V}_{k,h} (s) = \widetilde{Q}_{k,h} (s,\pi_{k,h}(s))$ \label{algline:opt_end}
        \ENDFOR
        \STATE Unroll a trajectory in $M_k$ using $\pi_k$
        \STATE Update $\widehat{W}_{sa}$ and $\hat{\theta}_{sa}$ for observed samples.
    \ENDFOR
    \end{algorithmic}
\end{algorithm}
In Corollary~\ref{cor:lincomb}, the bound is worse by a factor of $\sqrt{H}$ when compared to the $\widetilde{\C{O}}(HS\sqrt{AKH})$ bound of UCRL2 for tabular MDPs ($d=1$). This factor is incurred while bounding the sum of confidence widths in eq.~\eqref{eqline:c-schwartz} (in UCRL2 it is $\C{O}(\sqrt{SAKH})$).

\subsubsection{Proof of Theorem~\ref{thm:regret}}
We provide the key lemmas used in the analysis with the complete proof in Appendix~\ref{app:GLM-ORL}. Here, we assume that transition probability estimates are valid with probability at least $1-\delta_p$ and reward estimates with $1-\delta_r$ for all $(s,a)$ for all episodes. We first begin by showing that the computed policy's value is optimistic.
\begin{lemma}[Optimism]
If all the confidence intervals as computed in Algorithm~\ref{alg:GLM-ORL} are valid for all episodes $k$, then for all $k$ and $h \in [H]$ and $s,a \in \Scal \times \Acal$, we have:
\[
\tilde{Q}_{k,h}(s,a) \geq Q^*_{k,h}(s,a)
\]
\end{lemma}
\begin{proof}
We show this via an inductive argument. For every episode, the lemma is true trivially for $H+1$. Assume that it is true for $h+1$. For $h$, we have:
\begin{align*}
    &\tilde{Q}_{k,h}(s,a) - Q^*_{k,h}(s,a) \\
    & = (\widehat{P}_k(s,a)^\top\tilde{V}_{k,h+1} + \hat{r}_k(s,a) + \varphi_{k,h}(s,a)) \wedge V_h^{\max} \\
    & \quad - P_k(s,a)^\top V^*_{k,h+1} - r_k(s,a)
\end{align*}
We use the fact that when $\tilde{Q}_{k,h}(s,a) = V^{\max}_h$, the lemma is trivially satisfied. When $\tilde{Q}_{k,h}(s,a) < V^{\max}_h$, we have:
\begin{align*}
    &\tilde{Q}_{k,h}(s,a) - Q^*_{k,h}(s,a) \\
    & = \hat{r}_k(s,a) - r_k(s,a) + \widehat{P}_k(s,a)^\top(\tilde{V}_{k,h+1} - V^*_{k,h+1})\\
    & \quad + \varphi_{k,h}(s,a)  - (P_k(s,a) - \widehat{P}_k(s,a))^\top V^*_{k,h+1}\\
    & \geq -|\hat{r}_k(s,a) - r_k(s,a)| + \varphi_{k,h}(s,a) \\
    & \quad - \|P_k(s,a) - \widehat{P}_k(s,a)\|_1 \|\tilde{V}_{k,h+1}\|_\infty \ge 0
\end{align*}
The last step uses the guarantee on confidence intervals and the inductive assumption for $h+1$. Therefore, the estimated $Q$-values are optimistic by induction. 
\end{proof}
Using this optimism guarantee, we can bound the instantaneous regret $\Delta_k$ in episode $k$ as: $V^*_{k,1}(s) - V^{\pi_k}_{k,1} (s) \leq \widetilde{V}_{k,1}(s) - V^{\pi_k}_{k,1} (s)$. With $\widetilde{V}$ as the upper bound,  we can bound the total regret with the following Lemma:
\begin{lemma}
In the event that the confidence sets are valid for all episodes, then with probability at least $1-SH\delta_1$, the total regret $R(K)$ can be bounded by
\begin{align}
    R(K) \le {} & SH\sqrt{K\log \tfrac{6\log 2K}{\delta_1}} \nonumber \\
    {} & + \sum_{k=1}^K \sum_{h=1}^H \cdot (2\varphi_{k,h}(s_{k,h},a_{k,h}) \wedge V^{\max}_h) \label{eq:backup_error}
\end{align}
\end{lemma}
The proof is given in the appendix. The second term in ineq.~\eqref{eq:backup_error} can now be bounded as follows:
\begin{align}
    & \sum_{k=1}^K \sum_{h=1}^H (2\varphi(s_{k,h},a_{k,h}) \wedge V^{\max}_h) \nonumber\\
    & \leq \sum_{k=1}^K \sum_{h=1}^H (2\xi^{(r)}_{k,s_{k,h},a_{k,h}} \wedge V^{\max}_h) \nonumber \\
    & \quad + \sum_{k=1}^K \sum_{h=1}^H (2V^{\max}_{h+1}\xi^{(p)}_{k,s_{k,h},a_{k,h}} \wedge V^{\max}_h)
    \label{eq:opt_breakup1}
\end{align}
We ignore the reward estimation error in eq.~\eqref{eq:opt_breakup1} as it leads to lower order terms. The second expression can be again bounded as follows:
\begin{align}
    & \sum_{k=1}^K \sum_{h=1}^H (2V_{h+1}^{\max}\xi^{(p)}_{k,s_{k,h},a_{k,h}} \wedge V^{\max}_h) \nonumber \\
    & \leq 2\sum_{k,h} V^{\max}_h \left(1 \wedge  \beta \sqrt{S\gamma_k(s_{k,h},a_{k,h})} \|x_k\|_{Z_{k,sa,h}^{-1}} \right) \label{eq:prob_gamma_error}
\end{align}
Using Lemma~\ref{lem:ell_pot}, we see that
\begin{align*}
    \gamma_k(s,a) \coloneqq {} & f_\Phi(k,\delta_p) + \frac{8\eta}{\alpha}\log \frac{\text{det} (Z_{k,sa})}{\text{det}(Z_{1,sa})} \\
    \leq {} & \frac{\eta \alpha}{2S} + f_\Phi(KH,\delta_p) + \frac{8\eta}{\alpha}\log \frac{\text{det} (Z_{K+1,sa})}{\text{det}(Z_{1,sa})}\\
    \leq {} & \frac{\eta \alpha}{2S} + f_\Phi(KH,\delta_p) + \frac{8\eta d}{\alpha}\log\left(1+\frac{KHR^2}{\lambda d}\right) 
\end{align*}
We use $f_\Phi(k,\delta_p)$ to refer to the $Z_k$ independent terms in eq.~\eqref{eq:ONS_CI}. Setting $\bar{\gamma}_K$ to the last expression guarantees that $\tfrac{2S\bar{\gamma}_K}{\eta \alpha} \ge 1$. We can now bound the term in eq.~\eqref{eq:prob_gamma_error} as:
\begin{align}
    & 2\beta V_1^{\max}\sqrt{\frac{2S\bar{\gamma}_K}{\eta \alpha}} \sum_{k,h} \left(1 \wedge \sqrt{\frac{\eta \alpha}{2}}\|x_k\|_{Z_{k,sa,h}^{-1}}\right) \nonumber\\
    & \leq 2\beta V_1^{\max}\sqrt{\frac{2S\bar{\gamma}_K KH}{\eta \alpha}} \sqrt{\sum_{k,h} \left(1 \wedge \frac{\eta \alpha}{2}\|x_k\|^2_{Z_{k,sa,h}^{-1}}\right)} \label{eqline:c-schwartz1}
\end{align}
Ineq.~\eqref{eqline:c-schwartz1} follows by using Cauchy-Schwarz inequality.
Finally, by using Lemma~\ref{lem:ell_pot} in Appendix~\ref{app:GLM-ORL}, we can bound the term as
\begin{align*}
&\sum_{k=1}^K \sum_{h=1}^H (2V_{h+1}^{\max}\xi^{(p)}_{k,s_{k,h},a_{k,h}} \wedge V^{\max}_h)\\
& = 4\beta V_1^{\max}\sqrt{\frac{2S\bar{\gamma}_K KH}{\eta \alpha}} \sqrt{2HSAd \log \Big(1 + \frac{KHR^2}{\lambda d} \Big) }
\end{align*}

Now, after setting the failure probabilities $\delta_1=\delta_p=\delta_r = \delta/(2SA+SH)$ and taking a union bound over all events, we get the total failure probability as $\delta$. Therefore, with probability at least $1-\delta$, we can bound the regret of \texttt{GLM-ORL} as
\begin{equation*}
    R(K) = \widetilde{\C{O}} \left( \left( \frac{\sqrt{d}\max_{s,a}\|W_{sa}^*\|_F}{\sqrt{\alpha}} + \frac{d}{\alpha} \right)\beta SH^2\sqrt{AK} \right)
\end{equation*}
where $\max_{s,a}\|W_{sa}^*\|_F$ is replaced by the problem dependent upper bound assumed to be known a priori.\footnote{An improved dependence on $\sum_{s,a}\|W^*_{sa}\|_F$ can be obtained instead of $S\max_{s,a}\|W^*_{sa}\|_F$ in the regret bound.}

\subsubsection{Mistake bound for \texttt{GLM-ORL}}
\label{sec:mistake}
The regret analysis shows that the total value loss suffered by the agent is sublinear in $K$, and therefore, goes to $0$ on average. However, this can still lead to infinitely many episodes where the sub-optimality gap is larger than a desired threshold $\epsilon$, given that it occurs relatively infrequently. It is still desirable, for practical purposes, to analyze how frequently can the agent incur such mistakes. Here, a mistake is defined as an episode in which the value of the learner's policy $\pi_k$ is not $\epsilon$-optimal, i.e., $V^*_k - V^{\pi_k}_k \geq \epsilon$. In our setting, we can show the following result.
\begin{theorem}[Bound on the number of mistakes]
For any number of episodes $K$, $\delta \in (0,1)$ and $\epsilon \in (0,H)$, with probability at least $1-\delta$, the number of episodes where \texttt{GLM-ORL}'s policy $\pi_k$ is not $\epsilon$-optimal is bounded by
\begin{equation*}
    \C{O}\left( \frac{dS^2AH^5 \log (KH)}{\epsilon^2} \left(\frac{d\log^2 (KH)}{\alpha} + S\right)\right)
\end{equation*}
ignoring $\C{O}(\text{poly} (\log \log KH))$ terms.
\end{theorem}
We defer the proof to Appendix~\ref{app:mistake}. Note that this term depends poly-logarithmically on $K$ and therefore increases with time. The algorithm doesn't need to know the value of $\epsilon$ and result holds for all $\epsilon$. This differs from the standard mistake bound style PAC guarantees where a finite upper bound is given. \cite{dann2019policy} argued that this is due to the non-shrinking nature of the constructed confidence sets. As such, showing such a result for CMDPs requires a non-trivial construction of confidence sets and falls beyond the scope of this paper.

\subsection{RANDOMIZED EXPLORATION FOR GLM-CMDP}
\label{sec:GLM-RLSVI}
Empirical investigations in bandit and MDP literature has shown that optimism based exploration methods typically over-explore, often resulting in sub-optimal empirical performance. In contrast, Thompson sampling based methods which use randomization during exploration have been shown to have an empirical advantage with slightly worse regret guarantees. Recently, \cite{russo2019worst} showed that even with such randomized exploration methods, one can achieve a worst-case regret bound instead of the typical Bayesian regret guarantees. In this section, we show that the same is true for GLM-CMDP where a randomized reward bonus can be used for exploration. We build upon their work to propose an RLSVI style method (Algorithm~\ref{alg:GLM-RLSVI}) and analyze its expected regret.
\begin{algorithm}[ht]
    \caption{\texttt{GLM-RLSVI}}
    \begin{algorithmic}[1]
    \label{alg:GLM-RLSVI}
    \STATE {\bfseries Input:}$\Scal,\Acal,H,\Phi,d,\C{W}$, $\lambda$
    \STATE $\overline{V}_{k,H+1}(s) = 0$ $\forall s \in \Scal, k \in \B{N}$
    \FOR{$k \leftarrow 1, 2, 3, \ldots$}
        \STATE Observe current context $x_k$
        \FOR{$s \in \Scal, a \in \Acal$}
            \STATE $\widehat{P}_k(\cdot|s,a) \leftarrow \nabla \Phi (\widehat{W}_{k,sa} x_k)$
            \STATE $\hat{r}_k(s,a) \leftarrow \dotp{\hat{\theta}_{k,sa}}{x_k} $
            \STATE Compute conf. intervals using eqns. \eqref{eq:CI_p}, \eqref{eq:CI_r}
        \ENDFOR
        \FOR{$h \leftarrow H,H-1,\cdots, 1$, and $s \in \Scal$}
            \FOR{$a \in \Acal$}
                \STATE $\varphi = (H-h) \overline{\xi}^{(p)}_{k,sa} + \overline{\xi}^{(r)}_{k,sa}$ \label{algline:var_rand}
                \STATE Draw sample $b_{k,h}(s,a) \sim N(0,SH\varphi)$ \label{algline:rand_bonus}
                \STATE $\overline{Q}_{k,h} (s,a) =  \widehat{P}_{k,sa}^\top\overline{V}_{k,h+1} + \hat{r}_k(s,a) + b_{k,h}(s,a)$
            \ENDFOR
            \STATE $\pi_{k,h}(s) = \argmax_a \overline{Q}_{k,h}(s,a)$
            \STATE $\overline{V}_{k,h} (s) = \overline{Q}_{k,h} (s,\pi_{k,h}(s))$
        \ENDFOR
        \STATE Unroll a trajectory in $M_k$ using $\pi_k$.
        \STATE Update $\widehat{W}_{sa}$ and $\hat{\theta}_{sa}$ for observed samples.
    \ENDFOR
    \end{algorithmic}
\end{algorithm}
The main difference between Algorithm~\ref{alg:GLM-ORL} and Algorithm~\ref{alg:GLM-RLSVI} is that instead of the fixed bonus $\varphi$ (Line~\ref{algline:opt_bonus}) in the former, $\texttt{GLM-RLSVI}$ samples a random reward bonus in Line~\ref{algline:rand_bonus} for each $(s,a)$ from the distribution $N(0,HS\varphi^2)$. The variance term $\varphi$ is set to a sufficiently high value, such that, the resulting policy is optimistic with constant probability. We use a slightly modified version of the confidence sets as follows:
\begin{align*}
    \overline{\xi}^{(p)}_{k,sa} \coloneqq {} & 2 \wedge \left(\beta \sqrt{S} \sqrt{\gamma_{k,sa}}\|x_k\|_{Z_{k,sa}^{-1}}\right) \\
    \overline{\xi}^{(r)}_{k,sa} \coloneqq {} & B_rR \wedge \left(\tau_{k,sa} \|x_k\|_{Z_{k,sa}^{-1}}\right)
\end{align*}
The algorithm, thus, generates exploration policies by using perturbed rewards for planning. Similarly to \cite{russo2019worst}, we can show the following bound for the expected regret incurred by $\texttt{GLM-RLSVI}$:
\begin{theorem}
\label{thm:rand_regret}
For any contextual MDP with given link function $\Phi$, in Algorithm~\ref{alg:GLM-RLSVI}, if the MDP parameters for $M_k$ are estimated using Algorithm~\ref{alg:GLM-ONS}, with reward bonuses $b_{k,h}(s,a) \sim N(0, SH\varphi_{k,h}(s,a))$ where $\varphi_{k,h}(s,a)$ is defined in Line.~\ref{algline:var_rand}, the algorithm satisfies:
\begin{align*}
    & \bar{R}(K) = \E\left[\sum_{k=1}^K V^*_k - V^{\pi_k}_k\right] \nonumber \\
    & \qquad = \widetilde{\C{O}}\left(\left( \frac{\sqrt{d}\max_{s,a}\|W_{sa}^*\|_F}{\sqrt{\alpha}} + \frac{d}{\alpha} \right)\beta \sqrt{H^7S^3AK}\right)
\end{align*}
\end{theorem}
The proof of the regret bound is given in Appendix~\ref{app:GLM-RLSVI}. Our regret bound is again worse by a factor of $\sqrt{H}$ when compared to the $\widetilde{\C{O}}(H^3S^{3/2}\sqrt{AK})$ bound from \cite{russo2019worst} for the tabular case. Therefore, such randomized bonus based exploration algorithms can also be used in the CMDP framework with similar regret guarantees as the tabular case.

\section{LOWER BOUND FOR GLM CMDP}
\label{sec:lowerbnd}
In this section, we show a regret lower bound by constructing a family of hard instances for the GLM-CMDP problem. We build upon the construction of \cite{osband2016lower} and \cite{jaksch2010near} for the analysis\footnote{The proof is deferred to the appendix due to space constraints.}:
\begin{theorem}
For any algorithm $\mathbf{A}$, there exists CMDP's with $S$ states, $A$ actions, horizon $H$ and $K \geq dSA$ for logit and linear combination case, such that the expected regret of $\mathbf{A}$ (for any sequence of initial states $\in$ $\C{S}^K$) after $K$ episodes is:
\begin{equation*}
    \B{E}[R(K;\mathbf{A},M_{1:K}, s_{1:K})] = \Omega(H\sqrt{dSAK})
\end{equation*}
\end{theorem}
The lower bound has the usual dependence on MDP parameters in the tabular MDP case, with an additional $\C{O}(\sqrt{d})$ dependence on the context dimension. Thus, our upper bounds have a gap of $\C{O}(H\sqrt{dS})$ with the lower bound even in the arguably simpler case of Example~\ref{ex:lincomb}.

\section{IMPROVED CONFIDENCE SETS FOR STRUCTURED SPACES}
\label{sec:conversion}

In Section~\ref{sec:ONS}, we derived confidence sets for $W^*$ for the case when it lies in a bounded set. However, in many cases, we have additional prior knowledge about the problem in terms of possible constraints over the set $\C{W}$. For example, consider a healthcare scenario where the context vector contains the genomic encoding of the patient. For treating any ailment, it is fair to assume that the patient's response to the treatment and the progression in general depends on a few genes rather than the entire genome which suggests a sparse dependence of the transition model on the context vector $x$. In terms of the parameter $W^*$, this translates as complete columns of the matrix being zeroed out for the irrelevant indices. Thus, it is desirable to construct confidence sets which take this specific structure into account and give more problem dependent bounds. 

In this section, we show that it is possible to convert a generic regret guarantee of an online learner to a confidence set. If the online learner adapts to the structure of $\C{W}$, we would get the aforementioned improvement. The conversion proof presented here is reminiscent of the techniques used in \cite{abbasi2012online} and \cite{jun2017scalable} with close resemblance to the latter. For this section, we use $X_t$ to denote the $t \times d$ shaped matrix with each row as $x_i$ and $C_t$ as $t \times S$ shaped matrix with each row $i$ being $(W_ix_i)^\top$\footnote{We again solely consider the estimation problem for a single $(s,a)$ pair and study a $t$-indexed online estimation problem.}. Also, set $\overline{W}_t \coloneqq Z_{t+1}^{-1}X_t^\top C_t$. Using a similar notation as before, we can give the following guarantee.
\begin{theorem}[Multinomial GLM Online-to-confidence set conversion] 
Assume that loss function $l_i$ defined in eq.~\eqref{eq:loss_fn} is $\alpha$-strongly convex with respect to $Wx$. If an online learning oracle takes in the sequence $\{x_i, y_i\}_{i=1}^t$, and produces outputs $\{W_i\}_{i=1}^t$ for an input sequence $\{x_i, y_i\}_{i=1}^t$, such that:
\begin{equation*}
    \sum_{i=1}^t l_i(W_i) - l_i(W) \leq B_t \quad \forall \, W \in \C{W}, t>0,
\end{equation*}
then with $\overline{W}_t$ as defined above, with probability at least $1-\delta$, for all $t \geq 1$, we have
\begin{equation*}
    \|W^* - \overline{W}_t\|_{Z_{t+1}}^2 \leq \gamma_t
\end{equation*}
where $\gamma_t \coloneqq \gamma'_t(B_t) + \lambda B^2S - (\|C_t\|_F^2 - \dotp{\overline{W}_t}{X_t^\top C_t})$,
\begin{equation*}
    \gamma'_t(B_t) \coloneqq  1+\tfrac{4}{\alpha}B_t + \tfrac{8}{\alpha^2} \log \Big( \tfrac{1}{\delta}  \sqrt{4 + \tfrac{8B_t}{\alpha} + \tfrac{16}{\alpha^4 \delta^2}}\Big).
\end{equation*}
\end{theorem}
The complete proof can be found in Appendix~\ref{app:conversion}. Note that, all quantities required in the expression $\gamma_t$ can be incrementally computed. The required quantities are $Z_{t}$ and $Z_{t}^{-1}$ along with $X_t^\top C_t$ which are incrementally updated with $O(\text{poly}(S,d))$ computation. Also, we note that this confidence set is meaningful when $B_t$ is poly-logarithmic in $t$ which is possible for strongly convex losses as shown in \cite{jun2017scalable}. The dependence on $S$ and $d$ is the same as the previous construction, but the dependence on the strong convexity parameter is worse.
\paragraph{Column sparsity of $W^*$} Similar to sparse stochastic linear bandit, as discussed in \cite{abbasi2012online}, one can use an online learning method with the group norm regularizer ($\|W\|_{2,1}$). Therefore, if an efficient online no-regret algorithm has an improved dependence on the sparsity coefficient $p$, we can get an $O(\sqrt{p\log d})$ size confidence set. This will improve the final regret bound to $\tilde{\C{O}}(\sqrt{pdT})$ as observed in the linear bandit case. To our knowledge, even in the sparse adversarial linear regression setting, obtaining an efficient and sparsity aware regret bound is an open problem.

\section{DISCUSSION}
Here, we discuss the obtained regret guarantees for our methods along with the related work. Further, we outline the algorithmic/analysis components which are different from the tabular MDP case and lead to interesting open questions for future work.
\label{sec:disc}
\begin{table*}[ht!]
    \centering
    \renewcommand{\arraystretch}{1.5}
    \begin{tabular}{|c|c|c|c|}
    \hline
        \bf{Algorithm} & \textbf{$R^{\text{Linear}}(K)$} & \textbf{$R^{\text{Logit}}(K)$} & \textbf{$P_x(\cdot|s,a)$ normalized}\\ 
        \hline
        Algorithm~1 \citep{abbasi2014online} & $\widetilde{\C{O}}(dH^3S^2A\sqrt{K})$ & \ding{55} & \ding{55} \\
        $\texttt{ORLC-SI}$ \citep{dann2019policy} & $\widetilde{\C{O}}(dH^2S^{3/2}\sqrt{AK})$ & \ding{55} & \ding{55} \\
        $\texttt{GLM-ORL}$ (this work) & $\widetilde{\C{O}}(dH^2S\sqrt{AK})$ & $\widetilde{\C{O}}(dH^2S^3\sqrt{AK})$ & \ding{51}\\
    \hline
    \end{tabular}
    \caption{Comparison of regret guarantees for CMDPs. Last column denotes whether the transition dynamics $P_x(\cdot|s,a)$} are normalized in the model or not.
    \label{tab:opt_comparison}
\end{table*}
\subsection{RELATED WORK}
\paragraph{Contextual MDP} To our knowledge, \cite{hallak2015contextual} first used the term contextual MDPs and studied the case when the context space is finite and the context is not observed during interaction. They propose $\texttt{CECE}$, a clustering based learning method and analyze its regret. \cite{modi2018markov} generalized the CMDP framework and proved the PAC exploration bounds under smoothness and linearity assumptions over the contextual mapping. Their PAC bound is incomparable to our regret bound as a no-regret algorithm can make arbitrarily many mistakes $\Delta_k \ge \epsilon$ as long as it does so sufficiently less frequently. 

Our work can be best compared with \cite{abbasi2014online} and \cite{dann2019policy} who propose regret minimizing methods for CMDPs. \cite{abbasi2014online} consider an online learning scenario where the values $p_k(s'|s,a)$ are parameterized by a GLM. The authors give a no-regret algorithm which uses confidence sets based on \cite{abbasi2012online}. However, their next state distributions are not normalized which leads to invalid next state distributions. Due to these modelling errors, their results cannot be directly compared with our analysis. Even if we ignore their modelling error, in the linear combination case, we get an $\tilde{\C{O}}(S\sqrt{A})$ improvement. Similarly, \cite{dann2019policy} proposed an OFU based method $\texttt{ORLC-SI}$ for the linear combination case. Their regret bound is $\widetilde{\C{O}}(\sqrt{S})$ worse than our bound for $\texttt{GLM-ORL}$. In addition, the work also showed that obtaining a finite mistake bound guarantees for such CMDPs requires a non-trivial and novel confidence set construction. In this paper, we show that a $\text{polylog}(K)$ mistake bound can still be obtained. For a quick comparison, Table~\ref{tab:opt_comparison} shows the results from the two papers.

\paragraph{(Generalized) linear bandit} Our reward model is based on the (stochastic) linear bandit problem first studied by \cite{abe2003reinforcement}. Our work borrows key results from \cite{abbasi2011improved} for both the reward estimator and during analysis for the GLM case. Extending the linear bandit problem, \cite{filippi2010parametric} first proposed the generalized linear contextual bandit setting and showed a $\C{O}(d\sqrt{T})$ regret bound. We, however, leverage the approach from \cite{zhang2016online} and \cite{jun2017scalable} who also studied the logistic bandit and GLM Bernoulli bandit case. We extend their proposed algorithm and analysis to a generic categorical GLM setting. Consequently, our bounds also incur a dependence on the strong convexity parameter $\tfrac{1}{\alpha}$ of the GLM which was recently shown to be unavoidable by \cite{foster2018logistic} for proper learning in the closely related online logistic regression problem.

\paragraph{Regret analysis in tabular MDPs} \cite{auer2007logarithmic} first proposed a no-regret online learning algorithm for average reward infinite horizon MDPs, and the problem has been extensively studied afterwards. More recently, there has been an increased focus on fixed horizon problems where the gap between the upper and lower bounds has been effectively closed. \cite{azar2017minimax} and \cite{dann2019policy}, both provide optimal regret guarantees ($\widetilde{\C{O}}(H\sqrt{SAK})$) for tabular MDPs. Another series of papers \citep{osband2013more,osband2016generalization,russo2018tutorial} study Thompson sampling based randomized exploration methods and prove Bayesian regret bounds. \cite{russo2019worst} recently proved a worst case regret bound for $\texttt{RLSVI}$-style methods \citep{osband2016generalization}. The algorithm template and proof structure of $\texttt{GLM-RLSVI}$ is borrowed from their work.

\paragraph{Feature-based linear MDP}
\cite{yang2019sample} consider an RL setting where the MDP transition dynamics are low-rank. Specifically, given state-action features $\phi(s,a)$, they assume a setting where $p(s'|s,a) \coloneqq \sum_{i=1}^d \phi_i(s,a)\nu_i(s')$ where $\nu_i$ are $d$ base distributions over the state space. This structural assumption guarantees that the $Q^\pi(s,a)$ value functions are linear in the state-action features for every policy. \cite{yang2019reinforcement,jin2019provably} have recently proposed regret minimizing algorithms for the linear MDP setting. Although, their algorithmic structure is similar to ours (linear bandit based bonuses), the linear MDP setting is only superficially related to CMDP. In our case, the value functions are not linear in the contextual features for every policy and/or context. Thus, the two MDP frameworks and their regret analyses are incomparable.

\subsection{CLOSING THE REGRET GAP}
From the lower bound in Section~\ref{sec:lowerbnd}, it is clear that the regret bound of $\texttt{GLM-ORL}$ is sub-optimal by a factor of $\widetilde{\C{O}}(H\sqrt{dS})$. As mentioned previously, for episodic MDPs, \cite{azar2017minimax} and \cite{dann2019policy} propose minimax-optimal algorithms. The key technique in these analyzes is to directly build a confidence interval for the value functions and use a refined analysis using empirical Bernstein bonuses based on state-action visit counts saves a factor of $\C{O}(\sqrt{HS})$. In our case, we use a Hoeffding style bonus for learning the next state distributions to derive confidence sets for the value function. Further, the value functions in GLM-CMDP do not have a nice structure as a function of the context variable and therefore, these techniques do not trivially extend to CMDPs. Similarly, the dependence on context dimension $d$ is typically resolved by dividing the samples into phases which make them statistically independent \citep{auer2002using,chu2011contextual,li2017provably}. However, for CMDPs, these filtering steps cannot be easily performed while ensuring long horizon optimistic planning.
Thus, tightening the regret bounds for CMDPs is highly non-trivial and we leave this for future work.


\section{CONCLUSION AND FUTURE WORK}
\label{sec:conc}
In this paper, we have proposed optimistic and randomized no-regret algorithms for contextual MDPs which are parameterized by generalized linear models. We provide an efficient online Newton step (ONS) based update method for constructing confidence sets used in the algorithms. This work also outlines potential future directions: close the regret gap for tabular CMDPs, devise an efficient and sparsity aware regret bound and investigate whether a near-optimal mistake and regret bound can be obtained simultaneously. Lastly, extension of the framework to non-tabular MDPs is an interesting problem for future work.

\section*{Acknowledgements}
AM thanks Satinder Singh and Alekh Agarwal for helpful discussions. This work was supported in part by a grant from the Open Philanthropy Project to the Center for Human-Compatible AI, and in part by NSF grant CAREER IIS-1452099. AT would like to acknowledge the support of a Sloan Research Fellowship.

\bibliographystyle{apalike}
\bibliography{uai-main}
\newpage

\appendix
\section{PROOF OF THEOREM~\ref{thm:ONS}}
\label{app:ONS}
We closely follow the analysis from \cite{zhang2016online} and use properties of the categorical output space to adapt it to our case. The analysis is fairly similar, but carefully manipulating the matrix norms saves a factor of $\C{O}(S)$ in the confidence widths. For notation, we use $\nabla l_t (W_t)$ to refer to the derivative with respect to the matrix for loss $l_t$ and $\nabla l_t (W_t x_t)$ for the derivative with respect to the projection $W_t x_t$. $B_p$ denotes the upper bound on the $\ell_2$-norm of each row $W^{(i)}$ and $R$ is the assumed bound on the context norm $\|x\|_2$. Now, using the strong convexity of the loss function $l_t$ with respect to $W_t x_t$, for all $t$, we have:
\begin{align*}
l_t(W_t) - l_t(W^*) \leq {}&  \dotp{\nabla l_t(W_t x_t)}{W_t x_t - W^* x_t} \\
& - \tfrac{\alpha}{2} \underbrace{\|W^*x_t - W_t x_t\|^2_2}_{\coloneqq b_t}
\end{align*}
Taking expectation with respect to the categorical sample $y_t$, we get:
\begin{align}
    & 0 \le \E_{y_t}[l_t(W_t) - l_t(W^*)] \nonumber\\
    & \leq \E_{y_t}[\dotp{\nabla l_t(W_t x_t)}{W_t x_t - W^* x_t}] - \tfrac{\alpha}{2} b_t \nonumber\\
    & \leq  \E_{y_t}[\dotp{\nabla l_t(W_t x_t)}{W_t x_t - W^* x_t}] - \tfrac{\alpha}{2} b_t \label{eq:ineq}
\end{align}
where the \texttt{lhs} is obtained by using the calibration property from eq.~\eqref{eq:callib}. Now, for the first term on \texttt{rhs}, we have:
\begin{align}
    & \E_{y_t}[\dotp{\nabla l_t(W_t x_t)}{W_t x_t - W^* x_t}] \nonumber \\
    & = \E_{y_t}[\dotp{\nabla \Phi(W_t x_t) - y_t}{W_t x_t - W^* x_t}] \nonumber \\
    & = (\tilde{p}_t - p_t)^\top (W_t - W^*) x_t \nonumber \\
    & = \underbrace{(\tilde{p}_t - y_t)^\top (W_t - W^*) x_t}_{\coloneqq \mathbf{I}} \nonumber \\
    & \quad + \underbrace{(y_t - p_t)^\top (W_t - W^*) x_t}_{\coloneqq c_t} \label{eq:temp0}
\end{align}
where $\tilde{p}_t = \nabla \Phi(W_t x_t)$ and $\E[y_t] = p_t = \nabla \Phi(W^* x_t)$. We bound the term $\mathbf{I}$ using the following lemma:
\begin{lemma}
\begin{align}
    & \dotp{\nabla l_t(W_t x_t)}{W_t x_t - W^*x_t} \nonumber\\
    & \leq \frac{\|W_t - W^*\|_{Z_{t+1}}}{2\eta} - \frac{\|W_{t+1} - W^*\|_{Z_{t+1}}}{2\eta} \nonumber\\
    & \quad + 2 \eta \|x_t\|^2_{Z_{t+1}^{-1}} \label{eq:dist}
\end{align}
\end{lemma}
\begin{proof}
To prove this, we go back to the update rule in~\eqref{eq:ONS} which has the following form:
\begin{equation*}
    Y = \argmin_{W \in \C{W}} \frac{\|W-X\|_M^2}{2} + \eta a^\top W b
\end{equation*}
with $Y=W_{t+1}$, $X = W_t$, $a = \nabla l_t(W_t x_t) = \tilde{p}_t - y_t$, $b = x_t$ and $M = Z_{t+1}$.
For a solution to any such optimization problem, by the first order optimality conditions, we have:
\begin{eqnarray}
\dotp{(Y-X)M + \eta a b^\top}{W-Y} &\geq& 0\nonumber\\
(Y - X)MW &\geq& (Y - X)MY \nonumber \\
&&- \> \eta a^\top (W-Y)b \nonumber
\end{eqnarray}
Using this first order condition, we have
\begin{align}
    & \|X-W\|_M^2 - \|Y-W\|_M^2 \nonumber \\
    & = \sum_{i=1}^S X^iMX^i + W^iMW^i - Y^iMY^i \nonumber \\
    & \quad - W^iMW^i + 2(Y^i - X^i)MW^i \nonumber \\
    & \geq \|X-Y\|_M^2 - 2\eta a^\top (W-Y)b \nonumber \\
    & = \|X-Y\|_M^2 + 2\eta a^\top (Y-X)b \nonumber \\
    & \quad - 2\eta a^\top (W-X)b \nonumber \\
    & \geq \argmin_{A \in \R^{S \times d}} {\|A\|_M^2 + 2\eta a^\top A b} - 2\eta a^\top (W-X)b \label{eq:temp1}
\end{align}
Noting that $a = \tilde{p}_t - y^t$, we get
\begin{eqnarray}
\argmin_{A \in \R^{S \times d}} {\|A\|_M^2 + 2\eta a^\top A b} &\geq& \sum_{i=1}^S - \eta^2 a_i^2 \|b\|^2_{M^{-1}} \nonumber \\
&\geq& -4 \eta^2 \|b\|^2_{M^{-1}} \nonumber 
\end{eqnarray}
Substituting this and $W=W^*$ along with other terms in ineq.~\eqref{eq:temp1} proves the stated lemma (ineq.~\eqref{eq:dist}).
\end{proof}
Thus, from eqs.~\eqref{eq:ineq}, \eqref{eq:temp0} and \eqref{eq:dist}, we have
\begin{align}
    & \|W_{t+1} - W^*\|_{Z_{t+1}} \nonumber \\
    & \leq  \|W_t - W^*\|_{Z_t} - \frac{\eta \alpha}{2} b_t + 2 \eta c_t + 4\eta^2 \|x_t\|^2_{Z_{t+1}^{-1}} 
\end{align}
Bounding the first term on the \texttt{rhs} similarly, and telescoping the sum, we get:
\begin{align}
    & \|W_{t+1} - W^*\|_{Z_{t+1}} + \frac{\eta \alpha}{2}\sum_{i=1}^t b_i\nonumber \\
    & \leq \|W^*\|_{Z_1} + 2 \eta \sum_{i=1}^t c_i + 4\eta^2 \sum_{i=1}^t \|x_i\|^2_{Z_{i+1}^{-1}} \nonumber \\
    & \leq \lambda \|W^*\|^2_{F} + 2 \eta \sum_{i=1}^t c_i + 4\eta^2 \sum_{i=1}^t \|x_i\|^2_{Z_{i+1}^{-1}} \label{eq:temp_series}
\end{align}
We will now bound the sum $\sum_{i=1}^t c_i$ in ineq.~\eqref{eq:temp_series} using Bernstein's inequality for martingales in the same manner as \cite{zhang2016online}:
\begin{lemma}
With probability at least $1-\delta$, we have:
\begin{align}
    \sum_{i=1}^t c_i \leq 4B_pR + \frac{\alpha}{4}\sum_{i=1}^t b_i + \left(\frac{4}{\alpha} + \frac{8B_pR}{3}\right) \tau_t
\end{align}
where $\tau_t = \log (2 \lceil 2 \log St \rceil t^2/\delta)$.
\end{lemma}
\begin{proof}
The result can be easily derived from the proof of Lemma 5 in \cite{zhang2016online}. We provide the key steps here for completeness.

We first note that $c_t$ is a martingale difference sequence with respect to filtration $\C{F}_t$ induced by the first $t$ rounds including the next context $x_{t+1}$:
\begin{align*}
    & \E\left[(y_t-p_t)^\top (W_t-W^*)x_t|\C{F}_{t-1}\right] \\ 
    & = \E\left[(y_t-p_t)|\C{F}_{t-1}\right]^\top (W_t-W^*)x_t = 0
\end{align*}
Further, each term in this martingale series can be bounded as:
\begin{align*}
    |c_t| = {} & (y_t - p_t)^\top (W_t - W^*)x_t \nonumber \\
    \leq {} &  \|(y_t - p_t)\|_1 \|(W_t - W^*)x_t\|_\infty \nonumber \\
    \leq {} & 4B_pR \nonumber    
\end{align*}
Similarly, for martingale $C_t \coloneqq \sum_{i=1}^t c_i$, we bound the conditional variance as
\begin{align*}
    \Sigma_t^2 = {} & \sum_{i=1}^t \E_{y_i} \big[\big((y_t - p_t)^\top (W_t - W^*)x_t\big)^2\big] \nonumber \\
    \leq {} & \sum_{i=1}^t \E_{y_i} \big[\big(y_t^\top (W_t - W^*)x_t\big)^2\big] \nonumber \\
    \leq {} & \underbrace{\sum_{i=1}^t \|(W_t - W^*)x_t\|^2_2}_{\coloneqq A_t} \nonumber
\end{align*}
Thus, we have a natural upper bound for the conditional variance which is $\Sigma_t^2 \le 4B_p^2R^2St$. Now, consider two scanarios: \texttt{CASE I}: $A_t \ge 4B_p^2R^2/St$ and \texttt{CASE II}: $4B_p^2R^2/St \le A_t \le 4B_p^2R^2St$.

\texttt{CASE I}: Here, we directly bound the sum as
\begin{align*}
    C_t \le {} & \sum_{i=1}^t |c_i| \le {} 2\sum_{i=1}^t \|(W_t - W^*)x_t\|_2\\
    \le {} & 2\sqrt{t\sum_{i=1}^t \|(W_t - W^*)x_t\|^2_2} \le 4B_pR
\end{align*}
\texttt{CASE II}: We directly use the expression after applying Bernstein's inequality along with the peeling technique from \cite{zhang2016online}. Using that, we have:
\begin{align*}
    &P\left[ C_t \ge 2\sqrt{A_t\tau_t} + \frac{8B_pR\tau_t}{3} \right] \\
    & \le \sum_{j=-\log S}^m P\Big[ C_t \ge 2\sqrt{A_t\tau_t} + \frac{8B_pR\tau_t}{3}, \\
    & \qquad \qquad \frac{4B_pR^22^{j}}{t} \le A_t \le \frac{4B_pR^22^{j+1}}{t} \Big] \\
    & \le m'e^{-\tau_t}
\end{align*}
where $m=\log St^2$ and $m' = m+\log S = \log S^2t^2$. We set $\tau_t = \log \tfrac{2m't^2}{\delta}$, we get that with probability at least $1-\delta/2t^2$, we have:
\begin{align*}
    C_t \le 2 \sqrt{A_t\tau_t} + \frac{8B_pR\tau_t}{3}
\end{align*}
Taking a union bound over $t \ge 0$ and substituting $A_t = \sum_{i=1}^t b_i$, with probability at least $1-\delta$, for all $t \ge 0$, we get:
\begin{align*}
    \sum_{i=1}^t c_i \leq 4B_pR + 2 \sqrt{\tau_t \sum_{i=1}^t b_i} + \tfrac{8B_pR}{3}\tau_t
\end{align*}
Using the RMS-AM inequality, we get the desired expression:
\begin{align*}
    \sum_{i=1}^t c_i \leq 4B_pR + \frac{\alpha}{4} \sum_{i=1}^t b_i + \left( \frac{4}{\alpha} + \frac{8B_pR}{3}\right)\tau_t
\end{align*}
\end{proof}
Substituting the high probability upper bound over $\sum_{i=1}^t c_i$ in eq.~\eqref{eq:temp_series}, we get:
\begin{align}
    & \|W_{t+1} - W^*\|_{Z_{t+1}} \nonumber \\
    & \leq \lambda \|W^*\|_F^2 + 2\eta \Big[ 4B_pR + \Big( \frac{4}{\alpha} + \frac{8}{3}B_pR \Big)\tau_t \Big] \nonumber \\
    & \quad + 4\eta^2 \sum_{i=1}^t \|x_t\|^2_{Z_{t+1}^{-1}} \label{eq:temp3}
\end{align}
For getting the final result, we now bound the elliptic potential using the following Lemma from \cite{zhang2016online}:
\begin{lemma}[Lemma~6, \cite{zhang2016online}]
\begin{equation*}
    \sum_{i=1}^t \|x_t\|^2_{Z_{t+1}^{-1}} \leq \frac{2}{\eta \alpha} \log \frac{\det (Z_{t+1})}{\det(Z_1)}
\end{equation*}
\end{lemma}

\section{REGRET ANALYSIS}
\subsection{PROOF OF THEOREM \ref{thm:regret}}
\label{app:GLM-ORL}
We now provide a complete proof of Theorem~\ref{thm:regret}. 
\subsubsection{Failure events and bounding failure probabilities}
To begin with, we write the important failure events for the algorithm $F = F^{(r)} \cup F^{(p)} \cup F^{(O)}$ where each sub-event is defined as follows:
\begin{align*}
    & F^{(O)} \coloneqq \Big\{\exists K \in \B{N}: \sum_{k=1}^K \sum_{h,s,a} \Big(\B{P}_k[s_h,a_h=s,a|s_{k,1}] \\
    & \qquad - \B{I}[s_{k,h} = s, a_{k,h} = a] \Big) \ge SH\sqrt{K \log \tfrac{6 \log (2K)}{\delta_1}} \Big\}
\end{align*}
\begin{align*}
    & F^{(p)} \coloneqq \Big\{\exists \, s \in \Scal,a \in \Acal, k \in \B{N}: \\
    & \qquad \|W_{sa} - \widehat{W}_{k,sa}\|_{Z_{k,sa}} \ge \sqrt{\gamma_{k,sa}} \Big\}
\end{align*}
\begin{align*}
    & F^{(r)} \coloneqq \Big\{\exists \, s \in \Scal,a \in \Acal, k \in \B{N}: \\
    & \qquad \|\theta_{sa} - \widehat{\theta}_{k,sa}\|_{Z_{k,sa}} \ge \zeta_{k,sa} \Big\}
\end{align*}
\label{lem:failure}
Using high-probability guarantees for parameter estimation and concentration of measure, we have the guarantee that:
\begin{lemma}
The probabilities for failure events $F^{(O)}, F^{(p)}$ and $F^{(r)}$ are bounded bounded by $SH\delta_1$, $SA\delta_p$ and $SA\delta_r$ respectively.
\end{lemma}
\begin{proof}
The guarantee for $F^{(p)}$ follows from Theorem~\ref{thm:ONS} in Section~\ref{sec:ONS}. The failure probability $P(F^{(r)})$ can be bounded by using Theorem~20.5 from \cite{lattimore_szepesvári_2020}.

Lastly, the failure probability $P(F^{(O)})$ is directly taken from Lemma~23 of \cite{dann2019policy}.
\end{proof}

\subsubsection{Regret incurred outside failure events}
\begin{lemma}[Optimism]
If all the confidence intervals as computed in Algorithm~\ref{alg:GLM-ORL} are valid for all episodes $k$, then outside of failure event $F$, for all $k$ and $h \in [H]$ and $s,a \in \Scal \times \Acal$, we have:
\[
\tilde{Q}_{k,h}(s,a) \geq Q^*_{k,h}(s,a)
\]
\end{lemma}
\begin{proof}
For every episode, the lemma is true trivially for $H+1$. Assume that it is true for $h+1$. For $h$, we have:
\begin{align*}
    &\tilde{Q}_{k,h}(s,a) - Q^*_{k,h}(s,a) \\
    & = (\widehat{P}_k(s,a)^\top\tilde{V}_{k,h+1} + \hat{r}_k(s,a) + \varphi_{k,h}(s,a)) \wedge V_h^{\max} \\
    & \quad - P_k(s,a)^\top V^*_{k,h+1} - r_k(s,a) \\
    & = \hat{r}_k(s,a) - r_k(s,a) + \widehat{P}_k(s,a)^\top(\tilde{V}_{k,h+1} - V^*_{k,h+1})\\
    & \quad + \varphi_{k,h}(s,a)  - (P_k(s,a) - \widehat{P}_k(s,a))^\top V^*_{k,h+1}\\
    & \geq -|\hat{r}_k(s,a) - r_k(s,a)| + \varphi_{k,h}(s,a) \\
    & \quad - \|P_k(s,a) - \widehat{P}_k(s,a)\|_1 \|\tilde{V}_{k,h+1}\|_\infty \ge 0
\end{align*}
In the second equality step, we use the fact that when $\tilde{Q}_{k,h}(s,a) = V^{\max}_h$, the requirement is trivially satisfied. When $\tilde{Q}_{k,h}(s,a) < V^{\max}_h$, the step follows by definition. The last step uses the guarantee on confidence intervals and the inductive assumption for $h+1$. Therefore, the estimated $Q$-values are optimistic by induction.
\end{proof}
Therefore, using the optimism guarantee, we can bound the instantaneous regret $\Delta_k$ in episode $k$ as: $V^*_{k,1}(s) - V^{\pi_k}_{k,1} (s) \leq \widetilde{V}_{k,1}(s) - V^{\pi_k}_{k,1} (s)$. Thus, we have:
\begin{align}
    & \Delta_k \leq \widetilde{V}_{k,1}(s) - V^{\pi_k}_{k,1} (s) \nonumber \\
    & \leq (\widehat{P}_k(s,a)^\top\widetilde{V}_{k,2} + \hat{r}_k(s,a) + \varphi) \wedge V^{\max}_{1} \nonumber\\
    & \quad  - P_k(s,a)^\top V^{\pi_k}_{k,2} - r_k(s,a)\nonumber\\
    & \leq (\varphi + \widehat{P}_k(s,a) - P_k(s,a))^\top\widetilde{V}_{k,2} + \hat{r}_k(s,a)\nonumber\\
    & \quad - r_k(s,a)) \wedge V^{\max}_{1} + P_k(s,a)^\top (V^{\pi_k}_{k,2} - \widetilde{V}_{k,2})\nonumber\\
    & \leq 2\varphi \wedge V^{\max}_{1} + P_k(s,a)^\top (V^{\pi_k}_{k,2} - \widetilde{V}_{k,2})\nonumber\\
    & \leq \sum_{h,s,a} \Big[\B{P}_k[s_h,a_h=s,a|s_{k,1}] \nonumber\\
    & \qquad (2\varphi(s,a) \wedge V^{\max}_{h}) \Big] \label{eq:value_bnd}
\end{align}
Using Lemma~\ref{lem:failure}, we can show the following result:
\begin{lemma}
Outside the failure event $F^{(O)}$, i.e., with probability at least $1-SH\delta_1$, the total regret $R(K)$ can be bounded by
\begin{align}
\label{eq:visit_sum}
    R(K) \le {} & SH^2\sqrt{K\log \tfrac{6\log 2K}{\delta_1}} \nonumber \\
    {} & + \sum_{k=1}^K \sum_{h=1}^H \cdot (2\varphi_{k,h}(s_{k,h},a_{k,h}) \wedge V^{\max}_h)
\end{align}
\end{lemma}
\begin{proof}
\begin{align*}
    \Delta_k \leq {} & \sum_{h,s,a} \left[\B{P}_k[s_h,a_h=s,a|s_{k,1}] (2\varphi(s,a) \wedge V^{\max}_{h}) \right] \\
    \leq {} & \sum_{k=1}^K \sum_{h=1}^H \sum_{s,a} \Big(\B{P}_k[s_h,a_h=s,a|s_{k,1}] \\
    {} & - \B{I}_{k,h}(s,a) \Big) \left(2\varphi(s_{k,h},a_{k,h}) \wedge V^{\max}_h\right)\\
    {} & + \sum_{k=1}^K \sum_{h=1}^H \B{I}_{k,h}(s,a)(2\varphi(s_{k,h},a_{k,h}) \wedge V^{\max}_h)
\end{align*}
where $\B{I}_{k,h}(s,a)$ is the indicator function $\B{I}[s_{k,h}=s, a_{k,h}=a]$. From Lemma~\ref{lem:failure}, we know that the first term is bounded by $SH\sqrt{K\log \tfrac{6\log 2K}{\delta_1}}$ with probability at least $1-SH\delta_1$.
\end{proof}
Before bounding the second term in ineq.~\eqref{eq:visit_sum}, we state the following Lemma from \cite{abbasi2011improved} which is used frequently in our analysis:
\begin{lemma}[Determinant-Trace inequality]
\label{lem:ell_pot}
Suppose $X_1,X_2,\ldots,X_t \in \R^d$ and for any $1 \le s \le t$, $\|X_s\|_2 \le L$. Let $V_t \coloneqq \lambda\mathbf{I} + \sum_{s=1}^t X_sX_s^\top$ for some $\lambda \ge 0$. Then, we have:
\begin{align*}
    \det (V_t) \le \left( \lambda + tL^2/d \right)^d
\end{align*}
\end{lemma}
The second term in ineq.~\eqref{eq:backup_error} can now be bounded as follows:
\begin{align}
    & \sum_{k=1}^K \sum_{h=1}^H (2\varphi(s_{k,h},a_{k,h}) \wedge V^{\max}_h) \nonumber\\
    & \leq \sum_{k=1}^K \sum_{h=1}^H (2\xi^{(r)}_{k,s_{k,h},a_{k,h}} \wedge V^{\max}_h) \nonumber \\
    & \quad + \sum_{k=1}^K \sum_{h=1}^H (2V^{\max}_{h+1}\xi^{(p)}_{k,s_{k,h},a_{k,h}} \wedge V^{\max}_h)
    \label{eq:opt_breakup1-app}
\end{align}
We ignore the reward estimation error in eq.~\eqref{eq:opt_breakup1-app} as it leads to lower order terms. The second expression can be again bounded as follows:
\begin{align}
    & \sum_{k=1}^K \sum_{h=1}^H (2V_{h+1}^{\max}\xi^{(p)}_{k,s_{k,h},a_{k,h}} \wedge V^{\max}_h) \nonumber \\
    & \leq 2\sum_{k,h} V^{\max}_h \left(1 \wedge  \beta \sqrt{S\gamma_k(s_{k,h},a_{k,h})} \|x_k\|_{Z_{k,sa,h}^{-1}} \right) \label{eq:prob_gamma_error-app}
\end{align}
Using Lemma~\ref{lem:ell_pot}, we see that
\begin{align*}
    \gamma_k(s,a) \coloneqq {} & f_\Phi(k,\delta_p) + \frac{8\eta}{\alpha}\log \frac{\text{det} (Z_{k,sa})}{\text{det}(Z_{1,sa})} \\
    \leq {} & \frac{\eta \alpha}{2S} + f_\Phi(KH,\delta_p) + \frac{8\eta}{\alpha}\log \frac{\text{det} (Z_{K+1,sa})}{\text{det}(Z_{1,sa})}\\
    \leq {} & \frac{\eta \alpha}{2S} + f_\Phi(KH,\delta_p) + \frac{8\eta d}{\alpha}\log\left(1+\frac{KHR^2}{\lambda d}\right) 
\end{align*}
We use $f_\Phi(k,\delta_p)$ to refer to the $Z_k$ independent terms in eq.~\eqref{eq:ONS_CI}. Setting $\bar{\gamma}_K$ to the last expression guarantees that $\tfrac{2S\bar{\gamma}_K}{\eta \alpha} \ge 1$. We can now bound the term in eq.~\eqref{eq:prob_gamma_error-app} as:
\begin{align}
    & 2\beta V_1^{\max}\sqrt{\frac{2S\bar{\gamma}_K}{\eta \alpha}} \sum_{k,h} \left(1 \wedge \sqrt{\frac{\eta \alpha}{2}}\|x_k\|_{Z_{k,sa,h}^{-1}}\right) \nonumber\\
    & \leq 2\beta V_1^{\max}\sqrt{\frac{2S\bar{\gamma}_K KH}{\eta \alpha}} \sqrt{\sum_{k,h} \left(1 \wedge \frac{\eta \alpha}{2}\|x_k\|^2_{Z_{k,sa,h}^{-1}}\right)} \label{eqline:c-schwartz}
\end{align}
Ineq.~\eqref{eqline:c-schwartz} follows by using Cauchy-Schwarz inequality. We now bound the elliptic potential inside the square root in ineq.~\eqref{eqline:c-schwartz}:
\begin{lemma}
\label{lem:ell_sum}
For any $K \in \B{N}$, we have:
\begin{align*}
    \sum_{k,h} \left(1 \wedge \frac{\eta \alpha}{2}\|x_k\|^2_{Z_{k,sa,h}^{-1}}\right) \leq 2H \sum_{s,a} \log \left(\frac{\det Z_{k+1,sa}}{\det Z_{k,sa}}\right)
\end{align*}
\end{lemma}
\begin{proof}
Note that, instead of summing up the weighted operator norm with changing values of $Z_{k,h}$ for each observed transition of a pair $(s,a)$, we keep the matrix same for all observations in an episode. Note that, $Z_k$ denotes the matrix at the beginning of episode $k$ and therefore, does not include the terms $x_k x_k^\top$. Thus, for any episode $k$:
\begin{align*}
    & \sum_{h=1}^H \left(1 \wedge \frac{\eta \alpha}{2}\|x_k\|^2_{Z_{k,sa,h}^{-1}}\right) \\
    & \leq 2\sum_{s,a} \sum_{h=1}^H \B{I}_{k,h}(s,a) \log \left(1 + \frac{\eta \alpha}{2}\|x_k\|^2_{Z_{k,sa}^{-1}}\right)\\
    & = 2\sum_{s,a} N_k(s,a) \log \left(1 + \frac{\eta \alpha}{2}\|x_k\|^2_{Z_{k,sa}^{-1}}\right) \\
    & \leq 2\sum_{s,a} N_k(s,a) \log \left(1 + N_k(s,a) \frac{\eta \alpha}{2}\|x_k\|^2_{Z_{k,sa}^{-1}}\right)\\
    & = 2H \sum_{s,a} \log \left(\frac{\det Z_{k+1,sa}}{\det Z_{k,sa}}\right)
\end{align*}
where in the last step, we have used the following:
\begin{align*}
Z_{k+1} = {} & Z_{k}^{1/2} \left(1 + \frac{\eta \alpha}{2} N_k Z_{k}^{-1/2} x_{k} x_{k}^{\top} Z_{k}^{-1/2}\right) Z_{k}^{1/2}
\end{align*}
and then bound the determinant ratio using
\begin{align*}
\det Z_{k+1} = {} & \det Z_{k} \left(1+N_k  \frac{\eta \alpha}{2}\|x_{k}\|^2_{Z_{k}^{-1}}\right)
\end{align*}

\end{proof}
Finally, by using Lemma~\ref{lem:ell_pot}, we can bound the term as
\begin{align*}
&\sum_{k=1}^K \sum_{h=1}^H (2V_{h+1}^{\max}\xi^{(p)}_{k,s_{k,h},a_{k,h}} \wedge V^{\max}_h)\\
& \le 4\beta V_1^{\max}\sqrt{\frac{2S\bar{\gamma}_K KH}{\eta \alpha}} \sqrt{2HSAd \log \Big(1 + \frac{KHR^2}{\lambda d} \Big) }
\end{align*}
Now, we set each individual failure probability$\delta_1=\delta_p=\delta_r = \delta/(2SA+SH)$. Upon taking a union bound over all events, we get the total failure probability as $\delta$. Therefore, with probability at least $1-\delta$, we can bound the regret of \texttt{GLM-ORL} as
\begin{equation*}
    R(K) = \widetilde{\C{O}} \left( \left( \frac{\sqrt{d}\max_{s,a}\|W_{sa}^*\|_F}{\sqrt{\alpha}} + \frac{d}{\alpha} \right)\beta SH^2\sqrt{AK} \right)
\end{equation*}
where $\max_{s,a}\|W_{sa}^*\|_F$ is replaced by the problem dependent upper bound assumed to be known apriori.

\subsection{PROOF OF THEOREM~\ref{thm:rand_regret}}
\label{app:GLM-RLSVI}
Our analysis will closely follow the proof from \cite{russo2019worst}. We start by writing the concentration result for estimating MDP $M_k$ by using Algorithm~\ref{alg:GLM-ONS} and the linear bandit estimators. For notation, we use $\widehat{M}_k$ to denote the MDP constructed using the estimates $\widehat{W}_k$ and $\widehat{\theta}_k$. The perturbed MDP used in the algorithm is denoted by $\overline{M}_k$ and $\widetilde{M}_k$ will denote an MDP constructed using another set of \emph{i.i.d.} reward bonuses as $\overline{M}_k$. Specifically, we have:
\begin{lemma}
Let $\C{M}_k$ be the following set of MDPs:
\begin{align*}
    & \C{M}_k \coloneqq \{(P',R'): \forall (h,s,a),\, |(R'(s,a) - R_k(s,a)) \\
    & \quad + \langle P'(s,a) - P_k(s,a),V_{k,h+1}\rangle| \le \varphi_{k,h}(s,a)\}
\end{align*}
where $\varphi^2_{k,h}(s,a) = (\beta \sqrt{S\gamma_{k,sa}}(H-h) + \zeta_{k,sa})\|x_k\|_{Z_{k,sa}^{-1}}$. If we choose $\delta_p = \delta_r = \pi^2/SA$, then, we have:
\begin{align*}
    \sum_{k \in \B{N}} P_k[\widehat{M}_k \notin \C{M}_k] \le {} & \frac{\pi^2}{6}
\end{align*}
\end{lemma}
\begin{proof}
The proof follows from the analysis in Appendix~\ref{app:GLM-ORL} where the union bound over all $(s,a)$ pairs gives the total failure probability to be $\frac{\pi^2}{6}$.
\end{proof}

Given the concentration result, Lemma~4 from \cite{russo2019worst} directly applies to the CMDP setting in the following form:
\begin{lemma}
Let $\pi^*_k$ be the optimal policy for MDP $M_k$. If $\widehat{M}_k \in \C{M}_k$ and reward bonuses $b_{k,h}(s,a) \sim N(0,HS\varphi_{k,h}^2(s,a))$, then we have
\begin{align*}
    P\left[v^{\pi_k}_{\overline{M}_k} \ge v^{\pi^*}_{M_k}|\C{H}_{k-1}\right] \ge \B{F}(-1)
\end{align*}
where $\widehat{M}_k$ is the estimated MDP, $\overline{M}_k$ is the MDP obtained after perturbing the rewards and $\B{F}(\cdot)$ is the cdf for the standard normal distribution.
\end{lemma}
In a similar fashion, the following result can also be easily verified:
\begin{lemma}
\label{lem:rand_decomp}
For an absolute constant $c=\B{F}(-1)^{-1} \le 6.31$, we have:
\begin{align*}
    & R(K) \coloneqq \E_{\text{Alg}}\left[ \sum_{k=1}^K v^*_k(s_{k,1}) - v^{\pi_k}_k(s_{k,1}) \right]\\
    & \leq (c+1)\B{E}\left[\sum_{k=1}^K \left|v^{\pi_k}_{\overline{M}_k} - v^{\pi_k}_{M_k}\right|\right]\\
    & \quad + c\B{E}\left[\sum_{k=1}^K \left|v^{\pi_k}_{\widetilde{M}_k} - v^{\pi_k}_{M_k}\right|\right] + H\frac{\pi^2}{6}
\end{align*}
\end{lemma}
We will now bound the first term on the $\texttt{rhs}$ of Lemma~\ref{lem:rand_decomp} to get the final regret bound. The second term can be bounded in the same manner. For each episode, the summand in the first term can be written as:

\begin{align}
    & v^{\pi_k}_{\overline{M}}(s_{k,1}) - v^{\pi_k}_{M_k}(s_{k,1}) \nonumber\\
    & = \Big| \B{E}\Big[\sum_{h=1}^H \Big( \langle P_k(s_{k,h},a_{k,h}) - \widehat{P}_k(s_{k,h},a_{k,h}), \overline{V}_{k,h+1} \rangle \nonumber \\
    & \quad + \hat{r}_k(s_{k,h},a_{k,h}) - r_k(s_{k,h},a_{k,h}) \nonumber\\
    & \quad + b_{k,h}(s_{k,h},a_{k,h})\Big)\Big| \C{H}_{k-1}\Big]\Big|\nonumber\\
    & \le \Bigg|\E\left[\sum_{h=1}^H \left\langle P_k(s_{k,h},a_{k,h}) - \widehat{P}_k(s_{k,h},a_{k,h}), \overline{V}_{k,h+1}\right\rangle \right] \nonumber\\
    & \quad + \E\left[\sum_{h=1}^H r_k(s_{k,h},a_{k,h})-\hat{r}_k(s_{k,h},a_{k,h})|\C{H}_{k-1} \right]\Bigg|\nonumber \\
    & \quad +  \E\left[\sum_{h=1}^H \left| b_{k,h}(s_{k,h},a_{k,h})\right| \big|\C{H}_{k-1}\right] \label{eq:rand_decomp2}
\end{align}
where $\overline{V}_{k,h+1}$ denotes the $h^{\text{th}}$-step value of policy $\pi_k$ in $\overline{M}_k$. We will now bound each term individually where we ignore the reward term and the variance component due to reward uncertainty as both lead to lower order terms. Specifically, we directly consider $\varphi^2_{k,h}(s,a) = 2\left(\beta \sqrt{S\gamma_{k,sa}}(H-h)\right)\|x_k\|_{Z_{k,sa}^{-1}}$. For the last expression in eq.~\eqref{eq:rand_decomp2}, we focus on the first and third terms (the reward bonuses lead to lower order terms in the final regret bound). 

\begin{lemma}
\label{lem:rand_bonus_sum}
We have:
\begin{align*}
    & \E\left[\sum_{k=1}^K\sum_{h=1}^H \left| b_{k,h}(s_{k,h},a_{k,h})\right| \big|\C{H}_{k-1}\right] \\ 
    & = \widetilde{\C{O}}\left( \left( \frac{\sqrt{d}\max_{s,a}\|W_{sa}^*\|_F}{\sqrt{\alpha}} + \frac{d}{\alpha} \right)\beta S^{3/2}H^{5/2}\sqrt{AK} \right)
\end{align*}
\end{lemma}
\begin{proof}
We write $b_{k,h}(s_{k,h},a_{k,h}) = \sqrt{HS}\varphi_{k,h}(s_{k,h},a_{k,h}) \xi_{k,h}(s_{k,h},a_{k,h})$ where $\xi_{k,h}(s_{k,h},a_{k,h}) \sim N(0,1)$. Therefore, by using Holder's inequality, we have:
\begin{align*}
    &\E\left[\sum_{k=1}^K\sum_{h=1}^H \left| b_{k,h}(s_{k,h},a_{k,h})\right| \big|\C{H}_{k-1}\right] \\
    & \le \E\left[\max_{k,h,s,a} \xi_{k,h}(s,a)\right]\E\left[\sum_{k=1}^K\sum_{h=1}^H \sqrt{HS}\varphi_{k,h}(s_{k,h},a_{k,h})\right] 
\end{align*}
By using (sub)-Gaussian maximal inequality, we know that
\begin{align}
    \E\left[\max_{k,h,s,a} \xi_{k,h}(s,a)\right] = \C{O}\left(\log(HSAK)\right) \label{eq:noise_sum1}
\end{align}
For the second expression, we have:
\begin{align}
    & \E\left[\sum_{k=1}^K\sum_{h=1}^H \sqrt{HS}\varphi_{k,h}(s_{k,h},a_{k,h})\right] \nonumber \\
    & \le \sqrt{HS}\E\left[\sum_{k=1}^K\sum_{h=1}^H \varphi_{k,h}(s_{k,h},a_{k,h})\right] \nonumber \\
    & \le 2H^{3/2}\sqrt{S}\E\left[\sum_{k=1}^K\sum_{h=1}^H 1 \wedge \left(\beta \sqrt{S} \sqrt{\gamma_{k,sa}}\|x_k\|_{Z_{k,sa}^{-1}}\right)\right] \nonumber 
\end{align}
where we used the definition of $\overline{\xi}^{(p)}_{k,h}$ used in Section~\ref{sec:GLM-RLSVI}. Using the upper bound above along with Lemmas~\ref{lem:ell_pot} and~\ref{lem:ell_sum}, we obtain the bound:
\begin{align}
    & \E\left[\sum_{k=1}^K\sum_{h=1}^H \sqrt{HS}\varphi_{k,h}(s_{k,h},a_{k,h})\right] \nonumber \\
    & = \C{O}\left(\beta H^{5/2}S^{3/2}\sqrt{\frac{dA\bar{\gamma}_K K}{\eta \alpha}} \sqrt{\log \Big(1 + \frac{KHR^2}{\lambda d} \Big) }\right) \label{eq:noise_sum2}
\end{align}
We get the final bound on the term by combining eqs.~\eqref{eq:noise_sum1} and \eqref{eq:noise_sum2}. 
\end{proof}

We now bound the first term in eq.~\eqref{eq:rand_decomp2}:
\begin{lemma}
\label{lem:rand_est_err}
With the ONS estimation method and the used randomized bonus, we have:
\begin{align*}
    & \E\left[\sum_{k,h} \left|\langle P_k(s_{k,h},a_{k,h}) - \widehat{P}_k(s_{k,h},a_{k,h}), \overline{V}_{k,h+1}\rangle \right| \right] \\
    & \qquad = \widetilde{\C{O}}\left(\left( \frac{\sqrt{d}\max_{s,a}\|W_{sa}^*\|_F}{\sqrt{\alpha}} + \frac{d}{\alpha} \right)\beta \sqrt{H^7S^3AK}\right)
\end{align*}
\end{lemma}
\begin{proof}
We first rewrite the expression:
\begin{align*}
    & \E\left[\sum_{k,h} \left|\left\langle P_k(s_{k,h},a_{k,h}) - \widehat{P}_k(s_{k,h},a_{k,h}), V_{k,h+1}\right\rangle \right| \right] \\
    & \le \E\left[\sum_{k,h} \|\epsilon^p_k(s_{k,h}, a_{k,h})\|_1 \|V_{k,h+1}\|_\infty \right]
\end{align*}
where $\epsilon^p_k(s_{k,h}, a_{k,h}) = P_k(s_{k,h},a_{k,h}) - \widehat{P}_k(s_{k,h},a_{k,h})$. Using Cauchy-Schwarz inequality, we rewrite this as:
\begin{align*}
    \sqrt{\E\left[\sum_{k,h} \|\epsilon^p_k(s_{k,h}, a_{k,h})\|^2_1\right]}\sqrt{\E\left[ \sum_{k,h}\|V_{k,h+1}\|^2_\infty\right]}
\end{align*}
For bounding the sum of values under the second square root, we can directly use the Lemma~8 from \cite{russo2019worst}:
\begin{align}
    \label{eq:randval_bound}
    \sqrt{\E\left[ \sum_{k,h}\|V_{k,h+1}\|^2_\infty\right]} = {} & \widetilde{\C{O}}(H^3\sqrt{SK})
\end{align}
For bounding the expected estimation error, we consider two events: $F^{(p)}$ when the confidence widths are incorrect and $(F^{(p)})^c$ when the confidence intervals are valid for all $(s,a)$, $k$ and $h$. Therefore, we have:
\begin{align*}
    & \E\left[\sum_{k,h} \|\epsilon^p_k(s_{k,h}, a_{k,h})\|^2_1\right]\\
    & = \E\left[\sum_{k,h} \|\epsilon^p_k(s_{k,h}, a_{k,h})\|^2_1|F^{(p)}\right]P(F^{(p)}) \\
    & \quad + \E\left[\sum_{k,h} \|\epsilon^p_k(s_{k,h}, a_{k,h})\|^2_1|(F^{(p)})^c\right]P((F^{(p)})^c)
\end{align*}
Setting $\delta_p=1/KH$, we can bound the sum under failure event to a constant. For the other term, we see that it is equivalent to:
\begin{align}
    & \E\left[\sum_{k,h} \|\epsilon^p_k(s_{k,h}, a_{k,h})\|^2_1|(F^{(p)})^c\right]P((F^{(p)})^c) \nonumber \\
    & \le \E\left[ \sum_{k,h} \left(1 \wedge  \beta \sqrt{S\gamma_k(s_{k,h},a_{k,h})} \|x_k\|_{Z_{k,sa,h}^{-1}} \right)^2 \right] \nonumber\\
    & \le \frac{2\beta^2S\bar{\gamma}_K}{\eta \alpha} \E\left[\sum_{k,h} \left(1 \wedge \frac{\eta \alpha}{2}\|x_k\|^2_{Z_{k,sa,h}^{-1}}\right)\right] \nonumber \\
    & = \widetilde{\C{O}}\left(\left( \frac{d\max_{s,a}\|W_{sa}^*\|^2_F}{\alpha} + \frac{d^2}{\alpha^2} \right)\beta^2S^2AH\right) \label{eq:rand_est_err}
\end{align}
Combining eqs.~\eqref{eq:randval_bound} and \eqref{eq:rand_est_err}, we get the desired result.
\end{proof}
The final regret guarantee can be obtained by adding terms from Lemma~\ref{lem:rand_bonus_sum} and Lemma~\ref{lem:rand_est_err}.

\section{PROOF OF MISTAKE BOUND FROM SECTION~\ref{sec:mistake}}
\label{app:mistake}
In order to prove the mistake bound, we need to bound the number of episodes where the policy's value is more than $\epsilon$-suboptimal. We start with inequality \eqref{eq:value_bnd}:
\begin{eqnarray*}
\lefteqn{V^*_{k,1}(s) - V^{\pi_k}_{k,1} (s)} \\
&\leq& \sum_{h,s,a} \B{P}_k[s_h,a_h=s,a|s_{k,1}]
(2\varphi_{k,h}(s,a) \wedge V^{\max}_{h})
\end{eqnarray*}
We note that if $\varphi_{k,h}(s,a) \leq \frac{\epsilon}{2H}$ for all $k,h$ and $(s,a)$, then we have
\begin{eqnarray*}
\lefteqn{V^*_{k,1}(s) - V^{\pi_k}_{k,1} (s)} \\
&\leq& \sum_{h,s,a} \B{P}_k[s_h,a_h=s,a|s_{k,1}]
\frac{\epsilon}{H}\\
&\leq& \epsilon
\end{eqnarray*}

In order to satisfy the constraint, we bound each error term as: $\xi^{(p)} \leq \frac{\epsilon}{4H^2}$ and $\xi^{(r)} \leq \frac{\epsilon}{4H}$.

We bound the number of episodes where this constraint is violated. For simplicity, we consider that the rewards are known and only consider the transition probabilities in the analysis:
\begin{align}
    & \sum_{k \in [K]} \B{I}\left[\exists (s,a) \text{ s.t. } \xi^{(p)}_{k,sa} \geq \frac{\epsilon}{4H^2}\right] \nonumber \\
    & \leq \sum_{k \in [K]}\sum_{s,a} \B{I}\left[\beta \sqrt{S} \sqrt{\gamma_{k,sa}}\|x_k\|_{Z_{k,sa}^{-1}} \geq \frac{\epsilon}{4H^2}\right] \nonumber\\
    & \leq \sum_{k \in [K]} \sum_{s,a} \frac{16 \beta^2 S H^4 \gamma_{k,sa}}{\epsilon^2} \|x_k\|^2_{Z_{k,sa}^{-1}}\\
    & \leq \frac{16 \beta^2 S H^4 \gamma_{K+1}}{\epsilon^2}\sum_{k \in [K]} \sum_{s,a} \|x_k\|^2_{Z_{k,sa}^{-1}} \nonumber\\
    & \leq \frac{16 \beta^2 H^4 \gamma_{K+1}}{\epsilon^2} \sum_{s,a} \sum_{k \in [K]} \|x_k\|^2_{Z_{k,sa}^{-1}} \label{eq:mistake1}
\end{align}
where in the intermediate steps, we have used the nature of the indicator function and the fact that minimum is upper bounded by the average. Assuming that $N_{k,sa}$ denotes the number of visits to pair $(s,a)$ in episode $k$, we rewrite the inner term as:
\begin{align*}
\|x_k\|^2_{Z_{k+1,sa}^{-1}} = {} & x_k^\top (Z_k + N_{k,sa} x_k x_k^\top)^{-1} x_k\\
= {} & x_k^\top Z_{k,sa} x_k - \frac{N_{k,sa} x_k^\top Z_{k,sa}^{-1} x_k x_k^\top Z_{k,sa}^{-1} x_k}{1+N_{k,sa} x_k^\top Z_{k,sa}^{-1} x_k} \\
= {} & \|x_k\|^2_{Z_{k,sa}^{-1}} - \frac{N_k \|x_k\|^4_{Z_{k,sa}^{-1}}}{1+N_{k,sa} \|x_k\|^2_{Z_{k,sa}^{-1}}}
\end{align*}
With this setup, we get:
\begin{align*}
\|x_k\|^2_{Z_{k,sa}^{-1}} = {} & \frac{\|x_k\|^2_{Z_{k+1,sa}^{-1}}}{1 - N_{k,sa} \|x_k\|^2_{Z_{k+1,sa}^{-1}}}\\
\leq {} & \frac{\lambda + H}{\lambda} \|x_k\|^2_{Z_{k+1,sa}^{-1}}\\
\leq {} & \frac{\lambda + H}{\lambda} \langle Z_{k+1,sa}^{-1}, N_{k,sa} x_kx_k^\top\rangle
\end{align*}
Using Lemma~11 from \cite{hazan2007logarithmic}, the inner sum in eq.~\eqref{eq:mistake1}, can be bounded as:
\begin{align*}
    \frac{\lambda + H}{\lambda} \sum_{k \in [K]} \|x_k\|^2_{Z_{k+1}^{-1}} \le d \log \left( \frac{R^2 KH}{\lambda} + 1\right)
\end{align*}
Combining all these bounds, we get:
\begin{align*}
    & \sum_{k \in [K]} \B{I}\left[\exists (s,a) \text{ s.t. } \xi^{(p)}_{k,sa} \geq \frac{\epsilon}{4H^2}\right] \\
    & \leq  \frac{16 (\lambda + H)\beta^2 dS^2AH^4 \gamma_{K+1}}{\lambda \epsilon^2}\log \left( \frac{R^2 KH}{\lambda} + 1\right)
\end{align*}
Noting that $\gamma_{K+1} = \C{O}\left(\frac{d\log^2 KH}{\alpha} + S\right)$, we get the final mistake bound as:
\begin{equation*}
    \C{O}\left( \frac{dS^2AH^5 \log KH}{\epsilon^2} \left(\frac{d\log^2 KH}{\alpha} + S\right)\right)
\end{equation*}
ignoring $\C{O}(\text{poly} (\log \log KH))$ terms.

\section{PROOF OF THE LOWER BOUND}
\label{app:lower}
\begin{proof}
We start with the lower bound from \cite{jaksch2010near} adapted to the episodic setting.
\begin{theorem}[\cite{jaksch2010near}, Thm.~5]
For any algorithm $\mathbf{A}'$, there exists an MDP $M$ with $S$ states, $A$ actions, and horizon $H$, such that for $K \geq dSA$, the expected regret of $\mathbf{A}$ after $K$ episodes is:
\begin{equation*}
    \B{E}[R(K;\mathbf{A}',s,M)] = \Omega(H\sqrt{SAK})
\end{equation*}
\end{theorem}
\begin{figure}[htpb]
    \centering
    \includegraphics[width = 0.65\columnwidth, keepaspectratio]{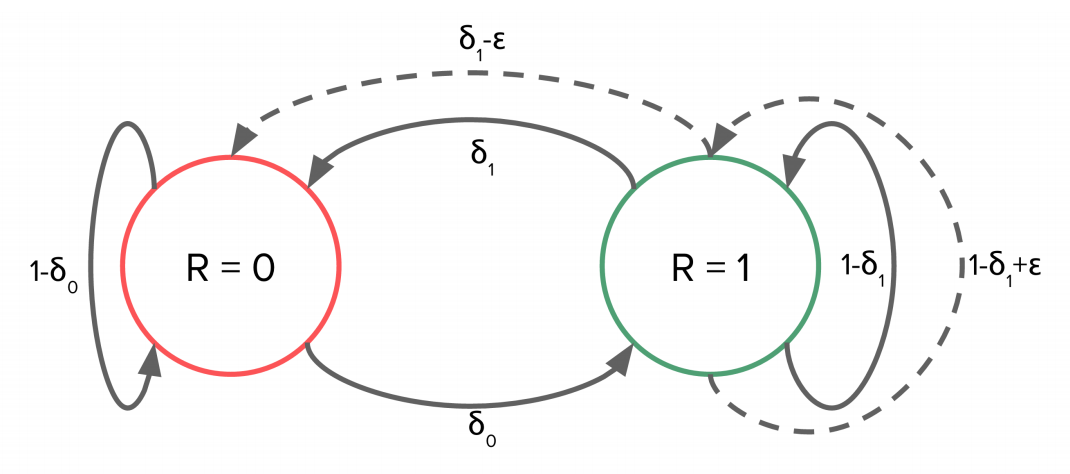}
    \caption{Hard 2-state MDP \citep{osband2016lower}}
    \label{fig:lower}
\end{figure}
The lower bound construction is obtained by concatenating $\lceil S/2 \rceil$-copies of a bandit-like 2-state MDP as shown in figure~\ref{fig:lower}\footnote{The two state MDP is built using $A/2$ actions with the rest used for concatenation. We ignore this as it only leads to a difference in constants.}. Essentially, state $1$ is a rewarding state and all but one action take the agent to state $0$ with probability $\delta_1$. The remaining optimal action transits to state $0$ with probability $\delta_1 - \epsilon$. This makes the construction similar to a hard Bernoulli multi-armed bandit instance which leads to the lower bound. Now, we will construct a set of such hard instances with the logit link function for transition probabilities. A similar construction for the linear combination case is discussed in Appendix~\ref{app:lower}. Since, the number of next states is 2, we use a GLM with parameter vector $w^*$ of shape $1 \times d$. Thus, for any context $x$, the next state probabilities are given as:
\begin{equation*}
    p(1|1,a;x) = \frac{\exp (w^{*}_a x)}{1+\exp (w^{*}_a x)} = \phi(w^{*}_a x)
\end{equation*}
If $w^*_a x = 0$, the value turns out to be $\frac{1}{2}$ which we choose as $\delta_1 - \epsilon$. For making the probability $\delta_1 = \frac{1}{2} + \epsilon$, we need to have $w^{*}_a x = \phi^{-1}(\delta_1) = c^*$. We consider the case where for each index $i$, all but one action has $w^*_{a}[i] = 0$ and one action $a^*_i$ has $w^*_{a^*}[i] = c^*$. The sequence of contexts given to the algorithm comprises of $K/d$ indicator vectors with $1$ at only one index. Therefore, for each episode $k$, we get an MDP with $p_k(0|1,a^*_{k \% d}) = 1/2$ for one optimal action and $1/2$ for all other actions. Therefore, this is a hard instance as shown in figure~\ref{fig:lower}. The agent interacts with each such MDP $K_i \approx K/d$ times. Further, these MDPs are decoupled as the context vectors are non-overlapping. Therefore, we have:
\begin{eqnarray*}
\lefteqn{\B{E}[R(K;\mathbf{A},M_{1:K}, s_{1:K})]}\\
&=& \sum_{i=1}^d \B{E} [R(K_i;\mathbf{A},M_{1:K}, s_{1:K})]\\
&\geq& \sum_{i=1}^d c H\sqrt{SAK/d} = c H\sqrt{dSAK}
\end{eqnarray*}
\end{proof}
\paragraph{Linear combination case} Similar to the logit case, we need to construct the sequence of hard instances in the linear combination case. It turns out that a similar construction works. Note that, in the linear combination case, each parameter vector $w^*_a$ now directly contains the probability of moving to the rewarding state. In other words, each index of this vector $w^*_a[i]$ corresponds to the next state visitation probability for the base MDP $M_i$. Therefore, for each index, we again set one action's value to $\frac{1}{2}+\epsilon$ and all others to 0. This maintains the independence argument and using indicator vectors as contexts, we get the same sequence of MDPs. The same lower bound can therefore be obtained for the linear combination case.

\section{OMITTED PROOFS FROM SECTION~\ref{sec:conversion}}
\label{app:conversion}
\begin{theorem}[Multinomial GLM Online-to-confidence set conversion] 
Assume that loss function $l_i$ defined in eq.~\eqref{eq:loss_fn} is $\alpha$-strongly convex with respect to $Wx$. If an online learning oracle takes in the sequence $\{x_i, y_i\}_{i=1}^t$, and produces outputs $\{W_i\}_{i=1}^t$ for an input sequence $\{x_i, y_i\}_{i=1}^t$, such that:
\begin{equation*}
    \sum_{i=1}^t l_i(W_i) - l_i(W) \leq B_t \quad \forall \, W \in \C{W}, t>0,
\end{equation*}
then with $\overline{W}_t$ as defined above, with probability at least $1-\delta$, for all $t \geq 1$, we have
\begin{equation*}
    \|W^* - \overline{W}_t\|_{Z_{t+1}}^2 \leq \gamma_t
\end{equation*}
where $\gamma_t \coloneqq \gamma'_t(B_t) + \lambda B^2S - (\|C_t\|_F^2 - \dotp{\overline{W}_t}{X_t^\top C_t})$,
\begin{equation*}
    \gamma'_t(B_t) \coloneqq  1+\tfrac{4}{\alpha}B_t + \tfrac{8}{\alpha^2} \log \Big( \tfrac{1}{\delta}  \sqrt{4 + \tfrac{8B_t}{\alpha} + \tfrac{16}{\alpha^4 \delta^2}}\Big).
\end{equation*}
\end{theorem}
\begin{proof}
Using the strong convexity of the losses $l_i$, we again have:
\begin{align*}
    & l_i(W_i) - l_i(W^*) \\
    & \geq \dotp{\nabla l_i(W^*)}{W^* - W_i} + \frac{\alpha}{2} \|W^*x_i - W_ix_i\|^2_2
\end{align*}
Summing this for $i=1$ to $t$ and substituting the regret bound $B_t$, we get
\begin{align}
    & \sum_{i=1}^t \|W^*x_i - W_ix_i\|^2_2 \nonumber \\
    & \leq \frac{2}{\alpha} B_t + \frac{2}{\alpha} \sum_{i=1}^t \dotp{p_t - y_t}{W^*x_i - W_i x_i} \label{eq:temp3_1}
\end{align}
Now, we focus on bounding the second term in the \texttt{rhs}. We note that for any $z \in \R^S$, we have
\begin{align*}
\dotp{p_t - y_t}{z} \leq \|p_t - y_t\|_2\|z\|_2 \leq 2\|z\|_2
\end{align*}
In addition, $\dotp{\eta_t}{z} \coloneqq \dotp{p_t - y_t}{z}$ is a martingale with respect to the filtration $\C{F}_{t} \coloneqq \sigma(x_1,y_1,\ldots,x_{t-1},y_{t-1},x_t)$. This shows that
\begin{align}
\E[D_t^\lambda|\C{F}_t] = \E[\exp (\lambda \dotp{\eta_t}{z} - \tfrac{1}{2}\lambda^2\|z\|_2^2)|\C{F}_t] \leq 1 \label{eq:supmartingale}
\end{align}
We can substitute $z_t = W^*x_t - W_tx_t$ which is $\C{F}_t$ measurable. Now, using $S_t = \sum_{i=1}^t \dotp{\eta_i}{z_i}$, ineq.~\eqref{eq:supmartingale} implies that $M_t^\lambda = \exp \big( 4\lambda S_t - \tfrac{1}{2}\lambda^2\sum_{i=1}^t\|z_i\|_2^2\big)$ is a $\C{F}_{t+1}$-adapted supermartingale. Using the same analysis as in \cite{abbasi2012online}, we get the following result:
\begin{corollary}[Corollary 8, \cite{abbasi2012online}]
With probability at least $1-\delta$, for all $t>0$, we have
\begin{align*}
    & \sum_{i=1}^t \dotp{\eta_i}{z_i} \\ 
    & \leq \sqrt{2\left( 1 + \sum_{i=1}^t \|z_i\|_2^2\right) \ln \left( \tfrac{1}{\delta} \sqrt{(1 + \sum_{i=1}^t \|z_i\|_2^2)}\right)}
\end{align*}
\end{corollary}
Substituting this in ineq.~\eqref{eq:temp3_1}, we get
\begin{align*}
    & \sum_{i=1}^t \|z_i\|^2_2 - \frac{2}{\alpha} B_t \\
    & \leq  \frac{2}{\alpha} \sqrt{2\left( 1 + \sum_{i=1}^t \|z_i\|_2^2\right) \ln \left( \tfrac{1}{\delta} \sqrt{(1 + \sum_{i=1}^t \|z_i\|_2^2)}\right)}
\end{align*}
We now use Lemma~2 from \cite{jun2017scalable}, to obtain a simplified bound:
\begin{lemma}[Lemma~2, \cite{jun2017scalable}]
For $\delta \in (0,1)$, $a \geq 0, f \geq 0, q \geq 1$, $q^2 \leq a + fq \sqrt{\log \tfrac{q}{\delta}}$ implies
\begin{align*}
    q^2 \leq 2a + f^2 \log \left( \frac{\sqrt{4a + f^4/(4\delta^2)}}{\delta} \right)
\end{align*}
\end{lemma}
With $q \coloneqq \sqrt{1 + \sum_{i=1}^t \|z_i\|_2^2}$, $a \coloneqq 1+\tfrac{2}{\alpha}B_t$ and $f = \tfrac{2\sqrt{2}}{\alpha}$, we now have:
\begin{align}
\sum_{i=1}^t \|W^*x_i - W_ix_i\|_2^2 \leq \gamma'_t
\label{eq:temp4}
\end{align}
with $\gamma'_t \coloneqq 1+\tfrac{4}{\alpha}B_t + \tfrac{8}{\alpha^2} \log \left( \tfrac{1}{\delta}  \sqrt{4 + \tfrac{8B_t}{\alpha} + \tfrac{16}{\alpha^4 \delta^2}}\right)$. 

We can rewrite ineq.~\eqref{eq:temp4} as
\begin{align}
    \|X_tW^{*\top} - C_t\|_F^2 \leq \gamma'_t
\end{align}
If we center this quadratic form around 
\begin{align*}
    \overline{W}_t \coloneqq {} & \argmin_W \|X_tW^\top - C_t\|_F^2 + \lambda \|W\|_F^2 \\
    = {} & Z_{t+1}^{-1}X_t^\top C_t
\end{align*}
we can rewrite the set as:
\begin{align*}
    & \|W^* - \overline{W}_t\|_{Z_{t+1}}^2 \nonumber \\
    & \leq  \lambda B_p^2S + \gamma'_t - \left(\|\overline{W}_t\|_F^2 + \|X_t\overline{W}_t^\top - C_t\|_F^2\right)
\end{align*}
Simplifying the expression on the \texttt{rhs} gives the stated result.
\end{proof}

\end{document}